\documentclass[letterpaper, 10 pt, conference]{ieeeconf}  

\IEEEoverridecommandlockouts                              

\usepackage{graphicx, graphics} 
\usepackage{times} 
\usepackage{amsmath} 
\usepackage{amssymb}  
\usepackage{color}
\usepackage{hyperref}

\usepackage{epstopdf}
\usepackage{cases}
\usepackage{amsthm}
\usepackage{comment}

\usepackage{real}

\usepackage{subcaption}

\title{\LARGE \bf
	From Bipedal Walking to Quadrupedal Locomotion:
	\\ Full-Body Dynamics Decomposition for Rapid Gait Generation 
}

\author{Wen-Loong Ma$^{1}$ 
        and Aaron D. Ames$^{2}$
\thanks{*This work is supported by NSF grants 1724464, 1544332 and 1724457.}
\thanks{$^{1}$W. Ma is with the department of mechanical engineering, $^{2}$A. D. Ames is with the faculty of the department of Control and Dynamical Systems, California Institute of Technology, Pasadena, CA, 91125. 
    {\tt\small wma, ames@caltech.edu}}. 
}

\begin{document}
\maketitle
\thispagestyle{empty}
\pagestyle{empty}


\begin{abstract}
This paper systematically decomposes a quadrupedal robot into bipeds to rapidly generate walking gaits and then recomposes these gaits to obtain quadrupedal locomotion. We begin by decomposing the full-order, nonlinear and hybrid dynamics of a three-dimensional quadrupedal robot, including its continuous and discrete dynamics, into two bipedal systems that are subject to external forces.  Using the hybrid zero dynamics (HZD) framework, gaits for these bipedal robots can be rapidly generated (on the order of seconds) along with corresponding controllers. The decomposition is achieved in such a way that the bipedal walking gaits and controllers can be composed to yield dynamic walking gaits for the original quadrupedal robot --- the result is the rapid generation of dynamic quadruped gaits utilizing the full-order dynamics. This methodology is demonstrated through the rapid generation (3.96 seconds on average) of four stepping-in-place gaits and one diagonally symmetric ambling gait at 0.35 m/s on a quadrupedal robot --- the Vision 60, with 36 state variables and 12 control inputs --- both in simulation and through outdoor experiments. This suggested a new approach for fast quadrupedal trajectory planning using full-body dynamics, without the need for empirical model simplification, wherein methods from dynamic bipedal walking can be directly applied to quadrupeds. 
\end{abstract}



\section{INTRODUCTION}
\label{sec:intro}

The control of quadrupedal robots has seen great experimental success in achieving locomotion that is robust and agile, dating back to the seminal work of Raibert \cite{raibertlegged}. These results have been achieved despite the fact that quadrupedal robots have more legs, degrees of freedom, and more complicated contact scenarios when compared to their bipedal counterparts. Bipedal robots (while seeing recent successes) still have yet to experimentally demonstrate the dynamic walking behaviors in real-world settings that quadrupeds are now displaying on multiple platforms.   Yet, due to the lower degrees of freedom and, importantly, simpler contact interactions with the world, gait generation for bipedal robots based upon the full-order dynamics has a level of rigor not yet present in the quadrupedal locomotion literature (which primarily leverages heuristic and reduced-order models). It is this gap between bipedal and quadrupedal robots that this paper attempts to address: can the formal full-order gait generation methods for bipeds be translated to quadrupeds while preserving the positive aspects quadrupedal locomotion? 


To achieve quadrupedal walking, controller design has widely adopted model-reduction techniques. For example, the massless leg assumption \cite{Bledt2018MIT, Bellicoso2017Dynamic}, linear inverted pendulum model \cite{LIPM_01, Avik17Modular} and assuming the 3D quadrupedal motion can be reduced to a planar motion \cite{Boussema2019Online, Avik15Parallel} are often utilized to mitigate the computational complexity of the quadrupedal dynamics so that online control techniques such as QP, MPC, LQR can be applied \cite{Boussema19Online}. While these methods are effective in practice, it often requires some add-on layers of parameter tuning due to the gap between model and reality. This tuning is particularly prevalent for bigger and heavier robots, whose ``ignored'' physical properties may play a more significant role. 

\begin{figure}[t]
	\centering
	\includegraphics[width=0.42\textwidth]{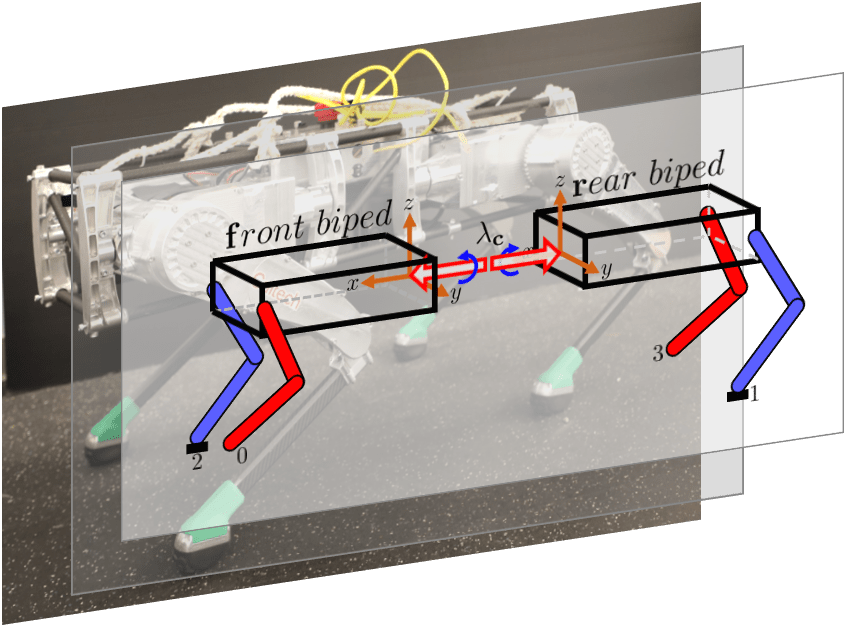}
	\vspace{-1.8mm}
	\caption{{\small 
	A conceptual illustration of the full body dynamics decomposition, where the 3D quadruped --- the Vision 60 --- is decomposed into two constrained 3D bipedal robots.}}
	\label{fig:2quad}
	\vspace{-5mm}
\end{figure}

In the context of bipedal robots, due to their inherently unstable nature, detailed model and rigorous controller design have been long been developed. A specific methodology that leverages the full-order dynamics of the robot to make formal guarantees is Hybrid Zero Dynamics (HZD) \cite{Westervelt2007a, Ames_RES_CLF_IEEE_TAC, Hamed_Buss_Grizzle_BMI_IJRR} which has seen success experimentally for both walking and running \cite{Sreenath_Grizzle_HZD_Walking_IJRR, ma2017bipedal, Reher2019DynamicWW}. 
A key to this success has been the recent developments in rapid HZD gait generation using collocation methods \cite{Hereid2018Rapid}, with the ability to generate gaits for high-dimensional robots in some cases in seconds \cite{hereid2016online}. 
Recently, the HZD framework was translated to quadrupedal robots both for gait generation and controller design \cite{ma2019First, Hamed2019Dynamically}. Although the end result was the ability generate walking, ambling and trotting for the full-order model, the high dimensional and complex contacts of the system made the gait generation complex with the fast gait being generated in 
$43$ seconds and hours of post-processing needed to guarantee stability.  
The goal of this paper is, therefore, to translate the positive aspects of HZD gait generation to quadrupeds while mitigating the aforementioned drawbacks.


%

Pioneers in robotics have discerned the correlation between bipedal and quadrupedal locomotion. For example, \cite{raibertlegged,Murphy1985} applied several bipedal gaits on quadrupedal robots; \cite{Avik15Parallel, Avik18Vertical} provided stability analysis for a planar abstract hopping robot. The ZMP condition of two bipeds was used to synthesis stability criteria for a quadruped in \cite{Laurenzi2018Quadrupedal}. However, these results rely on model reduction methods such as the 2D modeling and massless leg assumptions. Additionally, the focus was on composing bipedal controllers to stabilize quadrupedal locomotion rather than decomposing the dynamics of quadrupeds to bipedal systems while considering the evolution of the internal connection wrench. Notably, they lack a systematic approach of producing trajectories for the control of bipeds as a decomposed system from the quadrupedal robots.

The main contribution of this paper is the exact decomposition of quadrupeds into bipeds, wherein gaits can be rapidly generated and composed to be realized on the quadruped from which they were derived. Specifically, the main results of this paper are twofold: 
1) A systematic decomposition of the three-dimensional full body dynamics of a quadruped, which involved both the continuous and discrete dynamics, into two bipedal hybrid systems subject to external forces; 
2) An optimization algorithm that generates gaits for the bipedal system rapidly utilize the framework of HZD, wherein they can then be composed to yield gaits on the quadruped. 
The end result is that we are able to generate various bipedal gaits that can be recomposed to quadrupedal behaviors within seconds, and these behaviors are implemented successfully in simulation and experimentally in outdoor environments.

%

\begin{figure}[!t]
    \centering
	\vspace{3mm}
		\includegraphics[width=0.45\textwidth]{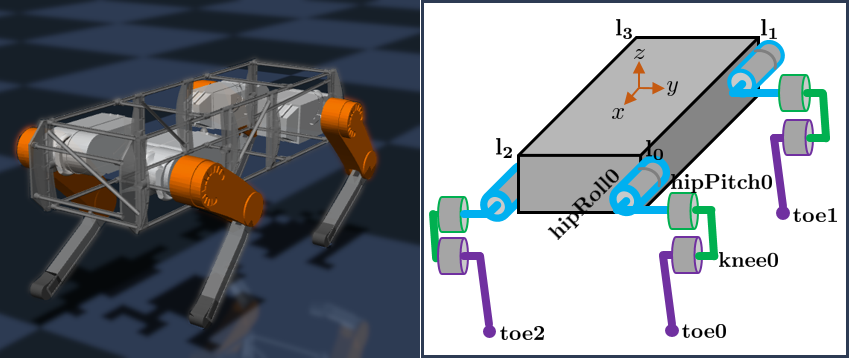}
		\vspace{-1.8mm}
		\caption{{\small 
		On the left is the the robot in MuJoCo, and on the right is the illustration of the configuration coordinates for the robot. The leg indices $\mathrm{l}_*$ are shown on the vertices of the \textit{body} link. Each leg has three actuated joints and equipped with a point contact toe.}}
		\label{fig:config}
	\vspace{-5mm}
\end{figure}

This paper is organized as follows: Section II introduces the general idea of decomposing the hybrid full-body dynamics of a quadrupedal robot into lower-dimensional half body dynamics of two identical bipeds. Based on this, we produced trajectories for stepping-in-place and ambling on a quadrupedal robot Vision60 in Section III. An analysis of its computation performance shows the efficiency compared against the full-body dynamics optimization for gait generation. In Section IV, we validate the resultant trajectories in MuJoCo \cite{Todorov2014Convex} (a commercial simulation environment), and five outdoor experiments to demonstrate the feasibility of these trajectories that are built based on decomposed bipedal dynamics. Section V concludes the paper and proposes several future directions. 

\section{Dynamics decomposition} \label{sec:DD}

In this section, we decompose the full body dynamics and control of quadrupedal robots into two identical bipedal systems. The nonlinear model of quadrupedal locomotion is a hybrid dynamical system, which is an alternating sequence of continuous- and discrete-time dynamics. The order of the sequence is dictated by contact events. 

\subsection{Quadrupedal Dynamics}
The full-body dynamics of quadrupedal robots have been detailed in \cite{ma2019First} and will be briefly revisited here to setup the problem properly. Note that in this section, we only focus on the most popular quadrupedal robotic behavior --- the \textit{diagonally supporting amble} (see \figref{fig:directed}). 

\subsubsection{State space and inputs}
The robot begin considered --- the Vision 60 V3.2 in \figref{fig:config} --- is composed of $13$ links: a \textit{body} link and $4$ limb links, each of which has three sublinks ---the \textit{hip}, \textit{upper} and \textit{lower} links. Utilizing the floating base convention \cite{Grizzle2014Models}, the configuration space is chosen as $q = ( q_b^\top, \theta_0^\top, \theta_1^\top, \theta_2^\top, \theta_3^\top)^\top \in\mathcal{Q}\subset\R^{18}$, where $q_b\in\R^3\times\mathrm{SO}(3)$ represents the Cartesian position and orientation of the \textit{body} linkage, and $\theta_i\in\R^3$ represents the three joints: \textit{hip roll, hip pitch and knee} on the leg $i\in\{0,1,2,3\}$. All of these leg joints are actuated, with torque inputs $u_i\in\R^3$. This yields the system's total DOF $n=18$ and control inputs $u =(u_0^\top,u_1^\top,u_2^\top,u_3^\top)^\top\in\R^{m},\ m=12$.
Further, we can define the state space $\mathcal{X}=T\mathcal{Q}\subseteq \R^{2n}$ with the state vector $x = (q^\top,\dot q^\top)^\top$, where $T\mathcal{Q}$ is the tangent bundle of the configuration space $\mathcal{Q}$.

\begin{figure}[!t]
	\centering
	\vspace{2mm}
		\includegraphics[width=0.303\textwidth]{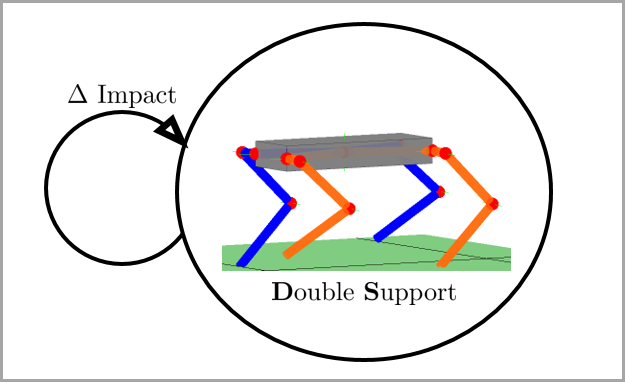}
		\vspace{-1.7mm}
		\caption{{\small The cyclic directed graph for the single-domain hybrid dynamics of the diagonally supporting ambling behavior.}}
		\label{fig:directed}
	\vspace{-5mm}
\end{figure}

\subsubsection{Continuous dynamics}
The continuous-time dynamics in \figref{fig:directed}, when toe1 and toe2 are on the ground, are modelled as constrained dynamics:
\begin{align}\begin{cases}
    D(q)\, \ddot q + H(q,\dot q) = Bu + J_1^\top(q)  \lambda_1 + J_2^\top(q)\, \lambda_2\\
    J_1(q)\, \ddot q  + \dot J_1(q,\dot q)\, \dot q = 0 \\
    J_2(q)\, \ddot q  + \dot J_2(q,\dot q)\, \dot q = 0 \\
    \end{cases}
    \label{eq:eom}
\end{align}
with the domain 
$
\mathcal{D} \defeq \{ x\in\mathcal{X} : \dot h_1(q,\dot q) = \dot h_2(q,\dot q) = 0, h_{z_1}(q)= h_{z_2}(q)=0\}.
$ 
In this formulation, we utilize the following notation: 
$D(q) \in\R^{n\times n}$ is the inertia-mass matrix; 
$H(q,\dot{q})\in\R^{n}$ contains Coriolis forces and gravity terms; 
$h_1(q), h_2(q)\in\R^3$ are the Cartesian positions of toe1 and toe2, their Jacobians are $J_*=\partial h_*/\partial q$; 
$h_{z_1}(q), h_{z_2}(q)$ are these toes' height; 
$\lambda_1, \lambda_2$ are the ground reaction force on toe1 and toe2; 
$B\in\R^{n\times m}$ is the actuation matrix. 
Essentially, we use a set of differential algebra equations (DAEs) to describe the dynamics of the quadrupedal robot that is subject to two holonomic constraints on toe1 and toe2.

\subsubsection{The discrete dynamics}
On the boundary of domain $\mathcal{D}$ we impose discrete-time dynamics to encode the \blue{perfectly inelastic} impact dynamics as toe0 and toe3 impact the ground (and suppressing the dependence of $D$ and $J_{*}$ on $q$ and $\dot{q}$):
\par\vspace{-3mm}{\small
\begin{align}
\begin{cases}
    D(\dot q^+ - \dot q^-) = J_0^\top\Lambda_0 + J_3^\top\Lambda_3 \\
    J_0 \dot q^+=0 \\ 
    J_3 \dot q^+=0 
\end{cases}\label{eq:impact}
\end{align}}%
by using conservation of momentum while satisfying the next domain's holonomic constraints, which is that toe0 and toe3 stay on the ground after the impact event. We denoted $\dot q^-$ and $\dot q^+$ as the pre- and pose-impact velocity terms, $\Lambda_0, \Lambda_3\in\R^3$ are the impulses exerted on toe0 and toe3.

\subsection{Continuous dynamics decomposition}
We now decompose the quadrupedal full body dynamics into two bipedal robots. First, as shown in \figref{fig:2quad}, the open-loop dynamics can be equivalently written as: 
\par\vspace{-3mm}{\small
\begin{numcases}{\text{OL-Dyn}\defeq}
    D_\rmf \ddot q_\rmf + H_\rmf = J_{\rmf_2}^\top \lambda_2 + B_\rmf u_\rmf -J_c^\top\lambda_c  \label{eq:ol21}\\
    J_{\rmf_2} \ddot q_\rmf + \dot J_{\rmf_2} \dot q_\rmf = 0                              \label{eq:ol22}\\
    D_\rmr \ddot q_\rmr + H_\rmr = J_{\rmr_1}^\top \lambda_1 + B_\rmr u_\rmr +J_c^\top\lambda_c  \label{eq:ol23}\\
    J_{\rmr_1} \ddot q_\rmr + \dot J_{\rmr_1} \dot q_\rmr = 0                              \label{eq:ol24}\\
    \ddot q_{\rmbr} - \ddot q_{\rmbf} = 0                                                  \label{eq:ol25}
\end{numcases}}%
wherein we utilized the following notation:
$q_{\rmbr}, q_{\rmbf}\in\R^3\times SO(3)$ are the coordinates for the body linkages of the front and rear bipeds (see \figref{fig:2quad}); 
$q_\rmf = (q_{\rmbf}^\top, \theta_0^\top, \theta_2^\top)^\top$ and $q_\rmr = (q_{\rmbr}^\top, \theta_1^\top, \theta_3^\top)^\top$ are the configuration coordinates for the front and rear bipeds; 
$D_\rmf(q_\rmf), D_\rmr(q_\rmr) \in\R^{12\times 12}$ are the inertia-mass matrices of the front and rear bipedal robots; 
The Jacobians $J_{\rmf_2} = \partial h_{\rmf_2} / \partial q_\rmf, J_{\rmr_1} = \partial h_{\rmr_1} / \partial q_\rmr$ with the Cartesian positions of toe2 --- $h_{\rmf_2}(q_\rmf)$ and toe1 --- $h_{\rmr_1}(q_\rmr)$; 
The Jacobian matrix for the connection constraint \eqref{eq:ol25} is $J_c = \partial (q_{\rmbr} - q_{\rmbf} ) / \partial q_{\rmf}$;
$u_\rmf=(u_0^\top, u_2^\top)^\top$ and $u_\rmr=(u_1^\top, u_3^\top)^\top$. 
Note that the Cartesian position of toe2 only depends on $q_\rmf$, which is due to the floating base coordinate convention.

\begin{proposition}
The dynamical system (OL-Dyn) is equivalent to the system \eqref{eq:eom}.
\end{proposition}
\begin{proof}
We can write \eqref{eq:ol21} and \eqref{eq:ol23} as:
\par\vspace{-3mm}{\small
\begin{align*}
    \begin{bmatrix}
            D_{\rmbf}      & D_{b_0} & D_{b_2} \\
            D_{b_0}^\top   & D_{0}   & 0       \\
            D_{b_2}^\top   & 0       & D_{2}  
    \end{bmatrix}
    \begin{bmatrix}   \ddot q_{\rmbf} \\ \ddot q_0 \\ \ddot q_2    \end{bmatrix}
    + \begin{bmatrix}    H_{\rmbf} \\ H_0 \\ H_2    \end{bmatrix}
    &= B_\rmf u_\rmf + J_{\rmf_2}^\top \lambda_2 - J_c^\top \lambda_c \\
    \vspace{2mm}
    \begin{bmatrix}
        D_{\rmbr}      & D_{b_1} & D_{b_3} \\
        D_{b_1}^\top   & D_{1}   & 0       \\
        D_{b_3}^\top   & 0       & D_{3}
    \end{bmatrix}
    \begin{bmatrix}   \ddot q_{\rmbr} \\ \ddot q_1 \\ \ddot q_3    \end{bmatrix}
    + \begin{bmatrix}    H_{\rmbr} \\ H_1 \\ H_3    \end{bmatrix}
    &= B_\rmr u_\rmr + J_{\rmr_1}^\top \lambda_1 + J_c^\top \lambda_c
\end{align*}}%
where each entry has a proper dimension to make the equations consistent. Expanding them yields: 
\par\vspace{-3mm}{\small
\begin{align*}
    \begin{bmatrix}
            D_{\rmbf}   & D_{b_0} & 0 & D_{b_2} & 0  \\
            D_{b_0}^\top   & D_{0}   & 0 & 0 & 0      \\
            0           & 0       & 0 & 0 & 0      \\
            D_{b_2}^\top   & 0       & 0 & D_{2}  & 0      \\
            0           & 0       & 0 & 0 & 0 
    \end{bmatrix}
    \hspace{-1mm}
    \begin{bmatrix}   \ddot q_{\rmbf} \\ \ddot q_0 \\\ddot q_1 \\\ddot q_2 \\ \ddot q_3    \end{bmatrix}
    \hspace{-1mm}+\hspace{-1mm}
    \begin{bmatrix}    H_{\rmbf} \\ H_0 \\ 0  \\ H_2  \\ 0  \end{bmatrix}
    \hspace{-1mm} &= \hspace{-1mm}
    \begin{bmatrix}    -\lambda_c \\ u_{0}  \\  0  \\ u_2  \\ 0  \end{bmatrix} 
    \hspace{-1mm}+\hspace{-1mm}
    J_2^\top \lambda_2,
\end{align*}
\begin{align*}
    \begin{bmatrix}
            D_{\rmbr}   & 0 & D_{b_1} & 0 & D_{b_3}  \\
            0           & 0 & 0       & 0 & 0        \\
            D_{b_1}^\top   & 0 & D_{1}   & 0 & 0        \\
            0           & 0           & 0 & 0 & 0    \\
            D_{b_3}^\top   & 0           & 0 & 0 & D_{3} 
    \end{bmatrix}\hspace{-1mm}
    \begin{bmatrix} \ddot q_{\rmbr} \\ \ddot q_0 \\\ddot q_1 \\\ddot q_2 \\ \ddot q_3    \end{bmatrix}
    \hspace{-1mm}+\hspace{-1mm}
    \begin{bmatrix} H_{\rmbr} \\ 0 \\ H_1  \\ 0  \\ H_3  \end{bmatrix}
    \hspace{-1mm}&= \hspace{-1mm} 
    \begin{bmatrix} \lambda_c \\ 0 \\ u_{1} \\ 0 \\ u_3  \end{bmatrix} 
    \hspace{-1mm}+\hspace{-1mm}
    J_1^\top\lambda_1.
\end{align*}}%
Combining these two equations, and using the fact that $q_{\rmbf} - q_{\rmbr} \equiv 0$\footnote{
$X\equiv Y$ means: $X(t)=Y(t)$ for all $t$ they are defined on. 
}
(holonomic constraint) yields the dynamics given in \eqref{eq:eom}. 
It is worthwhile to note that all the terms appeared in these equations can be verified using traditional rigid body dynamics and the corresponding details of the structure and necessary properties of the inertia-mass matrices can be found from the \textit{branch induced sparsity} \cite{Featherstone2008Rigid}.
\end{proof}

Note that \eqref{eq:ol25} can be equivalently replaced by summating the first $6$ equations of \eqref{eq:ol21} and \eqref{eq:ol23}:
\par\vspace{-3mm}{\small
\begin{align*}
    ( D_{\rmbf} + D_\rmbr ) \ddot q_{b_i} + \sum_{j=0}^3 D_{bj}\ddot q_j + H_\rmbf+H_\rmbr = 
    J^\top_{\rmr_1,b}\lambda_1 + J^\top_{\rmf_2,b}\lambda_2
\end{align*}}%
\vspace{-5mm}
\begin{align}
    \text{Denoted by:}\ \ \ 
    h_c(q_\rmf,\dot q_\rmf, \ddot q_\rmf, \lambda_2, q_\rmr,\dot q_\rmr, \ddot q_\rmr, \lambda_1) = 0
    \label{eq:hc}
\end{align}%
where $i=\rmf,\rmr$ and $J_{\rmr_1,b}, J_{\rmf_2,b}$ are the corresponding submatrices: 
$
J_{\rmr_1} = \begin{bmatrix}J_{\rmr_1,b} & J_{\rmr_1,\theta} \end{bmatrix}, ~ 
J_{\rmf_2} = \begin{bmatrix}J_{\rmf_2,b} & J_{\rmf_2,\theta} \end{bmatrix}.
$

Consider a system obtained from \eqref{eq:ol21}, \eqref{eq:ol22}, and \eqref{eq:hc} which defines the dynamics of the \textit{front biped} (see \figref{fig:2quad}):
\begin{align}
    \text{(f)} \defeq \begin{cases}
        D_\rmf \ddot q_\rmf + H_\rmf = J_{\rmf_2}^\top \lambda_2 + B_\rmf u_\rmf -J_c^\top\lambda_c \\
        J_{\rmf_2} \ddot q_\rmf + \dot J_{\rmf_2} \dot q_\rmf = 0  \\
         h_c(q_\rmf,\dot q_\rmf, \ddot q_\rmf, \lambda_2, q_\rmr,\dot q_\rmr, \ddot q_\rmr, \lambda_1) = 0
    \end{cases} 
\end{align} 
which is a dynamical system with feedforward terms $(q_\rmr,\dot q_\rmr, \ddot q_\rmr, \lambda_1)$. The dynamics of the \textit{rear biped} (r), can be similarly obtained using \eqref{eq:ol23}, \eqref{eq:ol24}, and \eqref{eq:hc}.  We have thus decomposed the dynamics of a quadrupedal robot \eqref{eq:eom} to two bipedal dynamical systems (f) and (r), as shown in \figref{fig:2quad}.

\begin{example}

\begin{figure}[t]
	\centering
	\includegraphics[width=0.3\textwidth]{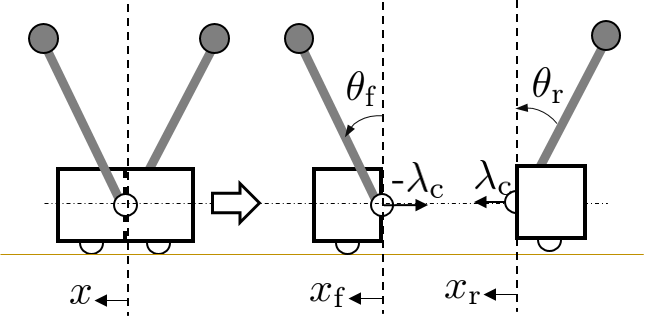}
	\vspace{-1.8mm}
	\caption{{\small The full body dynamics are composed of two invert pendulum carts, both rotational joints are actuated with inputs $u_\rmf, u_\rmr$. The mass of the cart is $2M$ and each of the pendulum weights $m$ with length $l$.}} 
	\label{fig:2IP}
	\vspace{-4mm}
\end{figure}

The idea of dynamics decomposition can be illustrated using a simple example in \figref{fig:2IP}. Note that each subsystem is not subject to any constraints. The half-body dynamics of a single cart with an inverted pendulum are:
\par\vspace{-3mm}{\small
\begin{align*}
    \begin{bmatrix}   M+m & -ml\cos\theta_i \\ -ml\cos\theta_i & ml^2    \end{bmatrix}
    \begin{bmatrix}   \ddot x_i \\ \ddot \theta_i    \end{bmatrix}
    + 
    \begin{bmatrix}   ml\dot\theta_i \sin\theta_i \\ -mgl\sin\theta_i    \end{bmatrix}
    =
    \begin{bmatrix}  \blue{\pm}\lambda_c \\ u_i \end{bmatrix}
\end{align*}}\noindent
where $i\in\{\rmf, \rmr\}$. The sign for $\lambda_c$ is negative for the front system and positive for the rear system. We can use a joint-space PD controller $u_\rmf(\theta_\rmf, \dot \theta_\rmf), u_\rmr(\theta_\rmr, \dot \theta_\rmr)$ to achieve a desired behavior such that the two invert pendulums vibrate symmetrically, i.e., $\theta_\rmf = -\theta_\rmr$. Then using \eqref{eq:hc} we have
$ ( 2M+2m ) \ddot x_i = 0 \Rightarrow \ddot x_i = 0 $,
which yields the internal connection force $\lambda_c = -ml\cos\theta_\rmf\ddot\theta_\rmf + ml\dot\theta_\rmf\sin\theta_\rmf$. In another word, when the two invert pendulums move symmetrically, both carts have zero acceleration. This physics example is rather trivial, but it suggested an insight on why a bipedal system (or a single invert pendulum) is difficult to stabilize while a quadrupedal system (or a parallel double invert pendula) is easy to remain stationary despite its higher DOF.


\end{example}

\subsection{Control decomposition}
We now design a control law to track the desired trajectories representing quadrupedal behaviors. The algorithm to produce these trajectories will be detailed in the next section. We define outputs \blue{(virtual constraints)} for the biped $i$ with $i\in\{\text{f}, \text{r}\}$ as $y_i = y^a_i(q_i) - \mathcal{B}_i(t)$, with $t$ the time and $\mathcal{B}(t)$ a $5$th order Be\'zier polynomial. For a simple case study, we chose $y^a_i(q_i)$ as the actuated joints: 
$
y^a_\rmf = (\theta_0^\top, \theta_2^\top)^\top, ~ 
y^a_\rmr = (\theta_1^\top, \theta_3^\top)^\top.
$
By imposing that the output dynamics of $y_i$ act like those of a linear system (as can be enforced through control law $u^c_i$), we have the closed-loop dynamics of the decomposed bipeds subject to control as follows:
\par\vspace{-3mm}{\small
\begin{align}
    \begin{cases}
        D_\rmf \ddot q_\rmf + H_\rmf = J_{\rmf_2}^\top \lambda_2 + B_\rmf u^c_\rmf -J_c^\top\lambda_c \\
        J_{\rmf_2} \ddot q_\rmf + \dot J_{\rmf_2} \dot q_\rmf = 0 \\
        \ddot y_\rmf = k_1\dot y_\rmf + k_2 y_\rmf \\
        h_c(q_\rmf,\dot q_\rmf, \ddot q_\rmf, \lambda_2, q_\rmr,\dot q_\rmr, \ddot q_\rmr, \lambda_1) = 0
    \end{cases} \label{eq:clsysf}
    \\
    \begin{cases}
        D_\rmr \ddot q_\rmr + H_\rmr = J_{\rmr_1}^\top \lambda_1 + B_\rmr u^c_\rmr +J_c^\top\lambda_c \\
        J_{\rmr_1} \ddot q_\rmr + \dot J_{\rmr_1} \dot q_\rmr = 0 \\
        \ddot y_\rmr = k_1\dot y_\rmr + k_2 y_\rmr \\
        h_c(q_\rmf,\dot q_\rmf, \ddot q_\rmf, \lambda_2, q_\rmr,\dot q_\rmr, \ddot q_\rmr, \lambda_1) = 0
    \end{cases} \label{eq:clsysr}
\end{align}
}%
In particular, the output dynamics implemented here is an implicit version of \textit{input-output feedback linearization}, details of this implementation can be found in \cite{hereid20163d, Ames_Human_Inspired_IEE_TAC}.

However, to design trajectories and determine the control inputs for a biped such as (f), we need to know all of the feedforward terms $(q_\rmr,\dot q_\rmr, \ddot q_\rmr, \lambda_1)$ for the time $t\in[0,T]$, with $T$ the time duration of a step. Therefore, the following equation is used to encode the desired correlation between the front and rear bipeds: 
\par\vspace{-3mm}{\small
\begin{align}
    \mathcal{B}_\rmr(t) = M \mathcal{B}_\rmf(t) + b.
    \label{eq:bi2quadmap}
\end{align}}%
Further, we consider a widely used motion of quadrupedal robots --- the \textit{diagonally symmetric} gait, where the joints of leg3 is a mirror of leg0 and those of leg1 is a mirror of leg2. In this case we have $M$ a diagonal matrix whose diagonal entries are $-1,1,1,-1,1,1$ and $b=0$. Note that one can specify other motions as well, for example, a torso-leaned motion can be achieved by offsetting $b$. 
Since the connection constraint $q_{\rmbf}\equiv q_{\rmbr}$ is always satisfied by mechanical wrenches $\lambda_c$, then on the \textit{zero dynamics (ZD) surface} \cite{Sastry1999Nonlinear}, i.e., $y_i(q_i) \blue{= y^a_i(q_i) - \mathcal{B}_i(t)} \equiv 0$, we have the following correlation between the two bipeds:
\begin{align}
    q_\rmr \equiv A q_\rmf + b,\ \ \text{where}\ A = \begin{bmatrix} I& \\ & M\end{bmatrix}.
\end{align}
Additionally, to determine $\lambda_1$ of the biped (r), we also need to impose the constraint \eqref{eq:ol24} to the system in \eqref{eq:clsysf}. 
Then subtract the dynamics of biped (f) from \eqref{eq:hc} 
to have the closed-loop dynamics of the front biped subject to the connection wrench $\lambda_c$ as: 
\par\vspace{-3mm}{\small
\begin{numcases}{\text{CL-Dyn-f}\defeq}
        D_\rmf \ddot q_\rmf + H_\rmf = J_{\rmf_2}^\top \lambda_2 + B_\rmf u^c_\rmf -J_c^\top\lambda_c  \label{eq:clf1}\\
        \ddot y_\rmf = k_1\dot y_\rmf + k_2 y_\rmf \equiv 0                              \label{eq:clf2}\\
        J_{\rmf_2} \ddot q_\rmf + \dot J_{\rmf_2} \dot q_\rmf = 0                        \label{eq:clf3}\\
        J_{\rmr_1} A\ddot q_\rmf + \dot J_{\rmr_1} A\dot q_\rmf = 0                      \label{eq:clf4}\\
        \hat D_\rmf A\ddot q_\rmf + \hat H_\rmr = J^\top_{\rmr_1,b}\lambda_1             \label{eq:clf5}
\end{numcases}}%
with $\hat D_\rmf\in\R^{6\times 12}, \hat H_\rmf\in\R^6$ the first 6 rows of $D_\rmf$ and $H_\rmf$, respectively. We now have the decomposed dynamics of system (f) that is independent from the feedforward terms. We can view this system as a dynamical system \eqref{eq:clf1} subject to \textit{virtual constraint} \eqref{eq:clf2} with inputs $u_\rmf^c$ and \textit{mechanical constraints} \eqref{eq:clf3}, \eqref{eq:clf4}, and \eqref{eq:clf5} with inputs $\lambda_1$, $\lambda_2$, and $\lambda_c$. 

\begin{figure*}[t]
\vspace{3mm}
	\begin{center}
		\includegraphics[width = 0.17\textwidth]{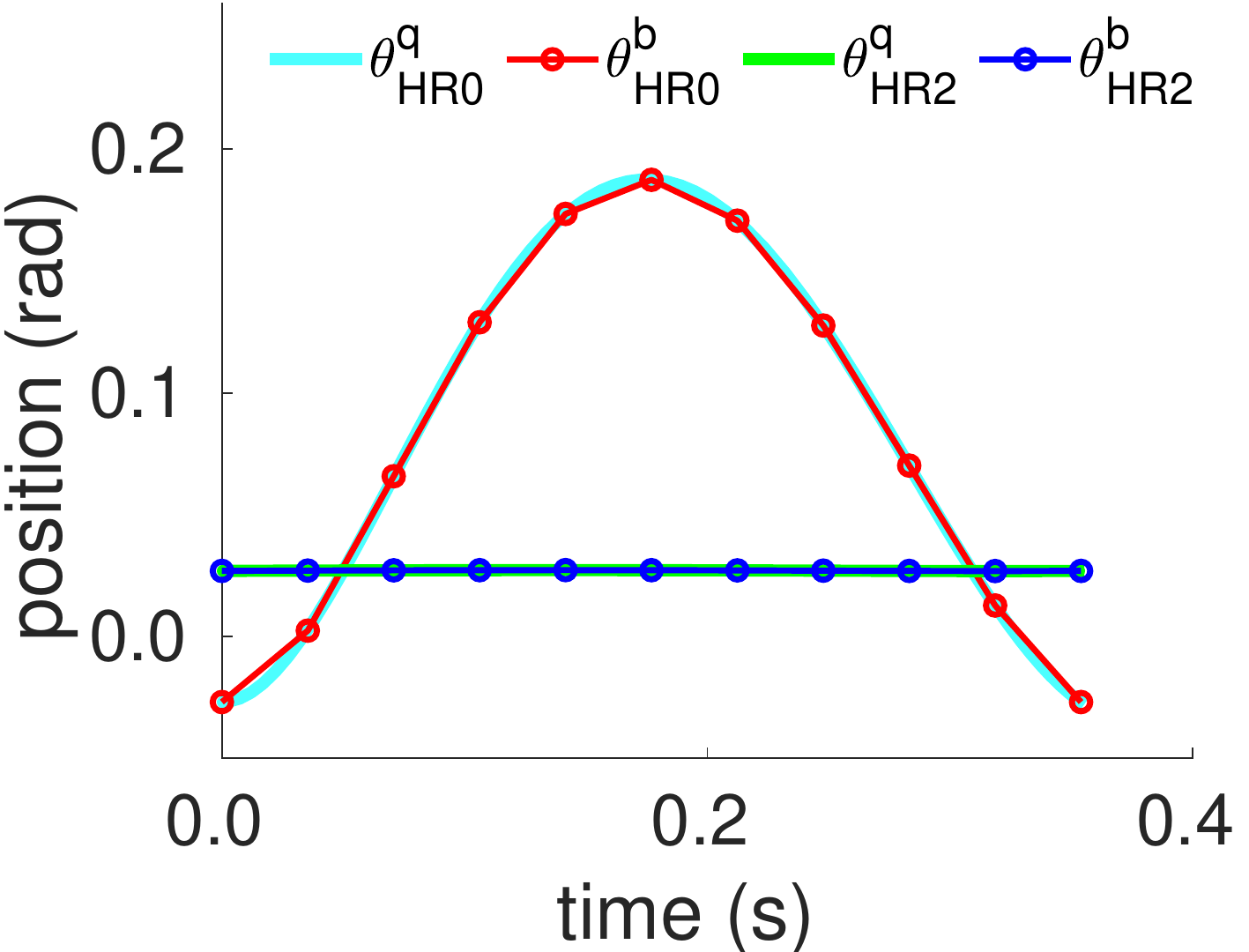}
		\includegraphics[width = 0.17\textwidth]{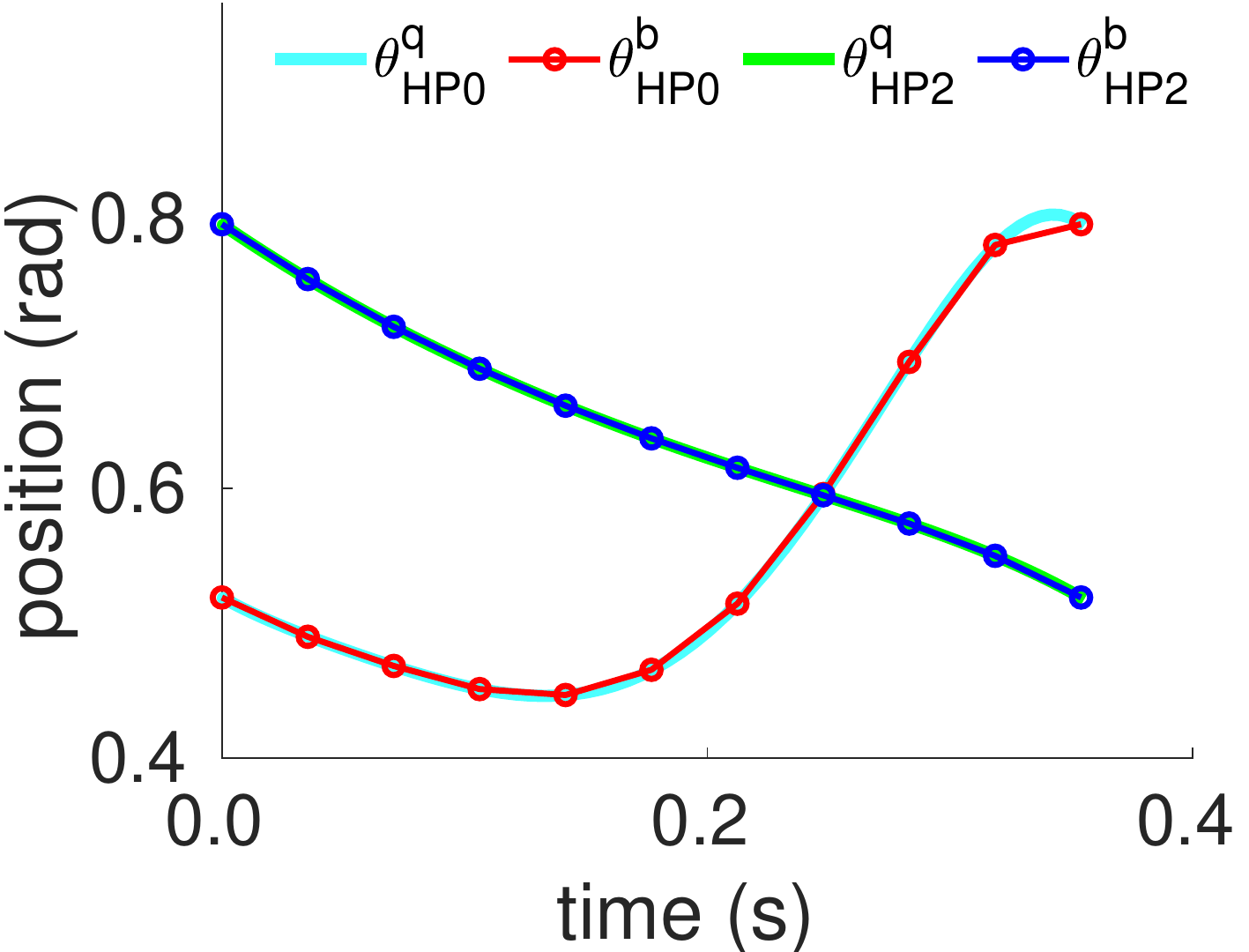}
		\includegraphics[width = 0.17\textwidth]{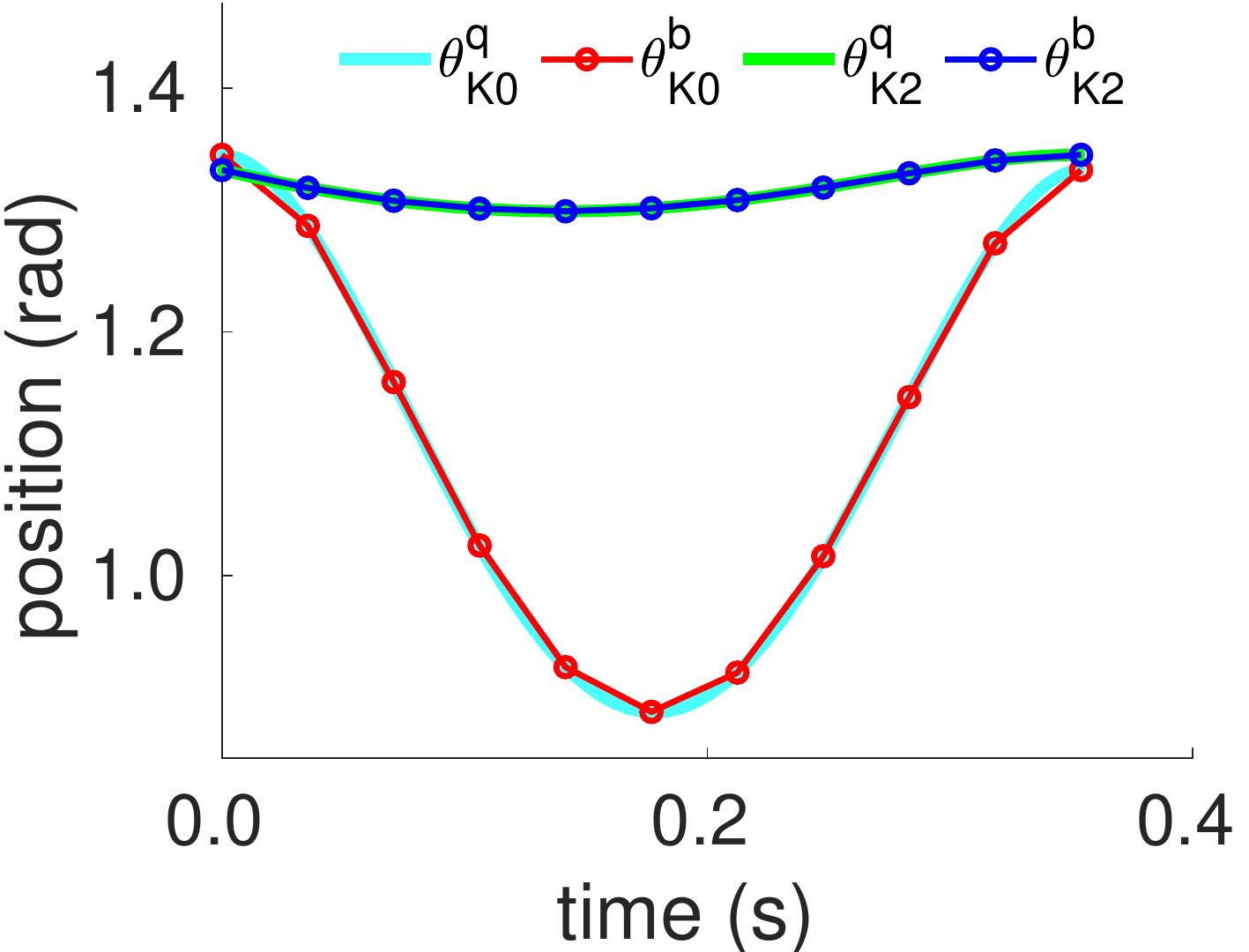}
		\includegraphics[width = 0.165\textwidth]{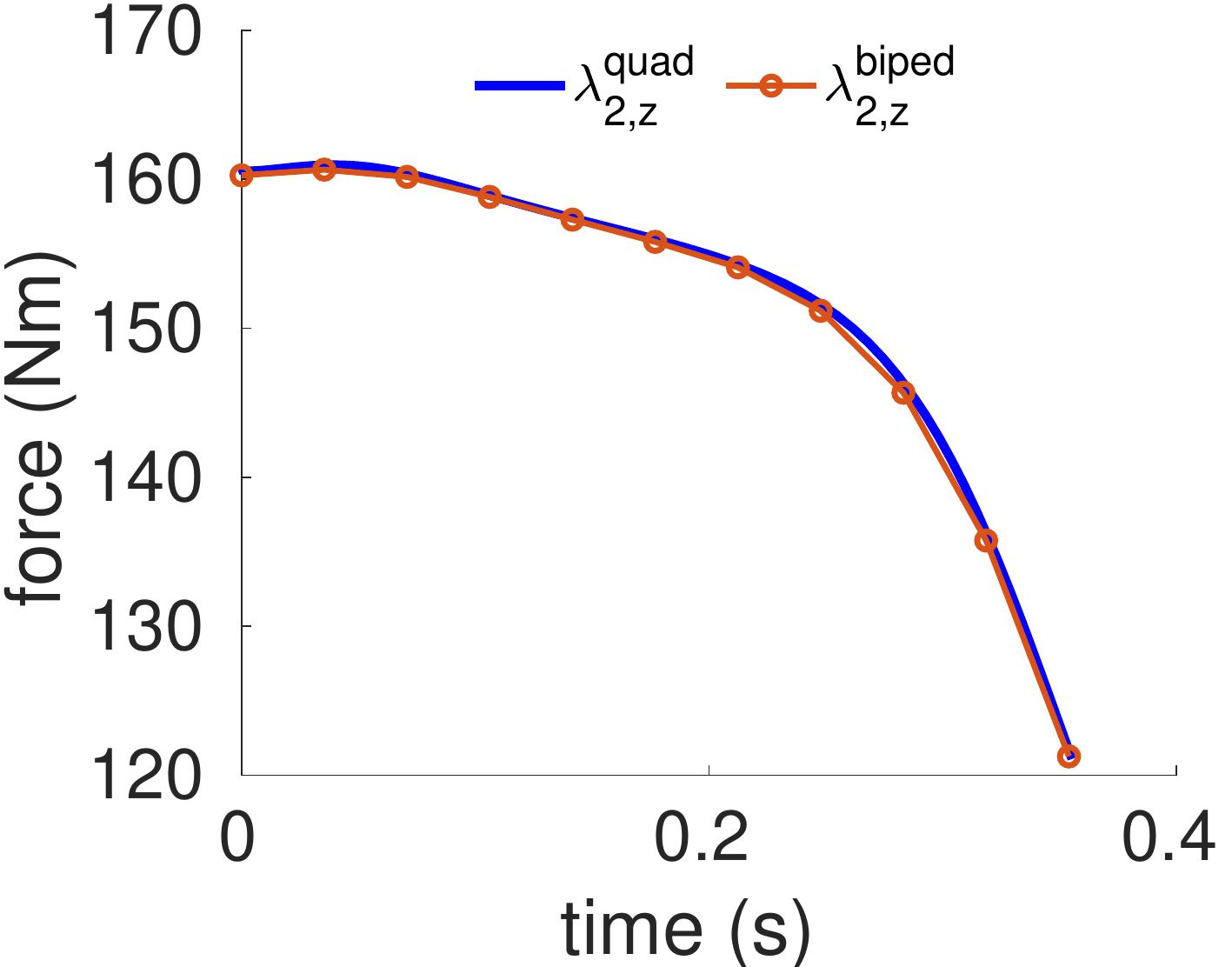}
		\includegraphics[width = 0.165\textwidth]{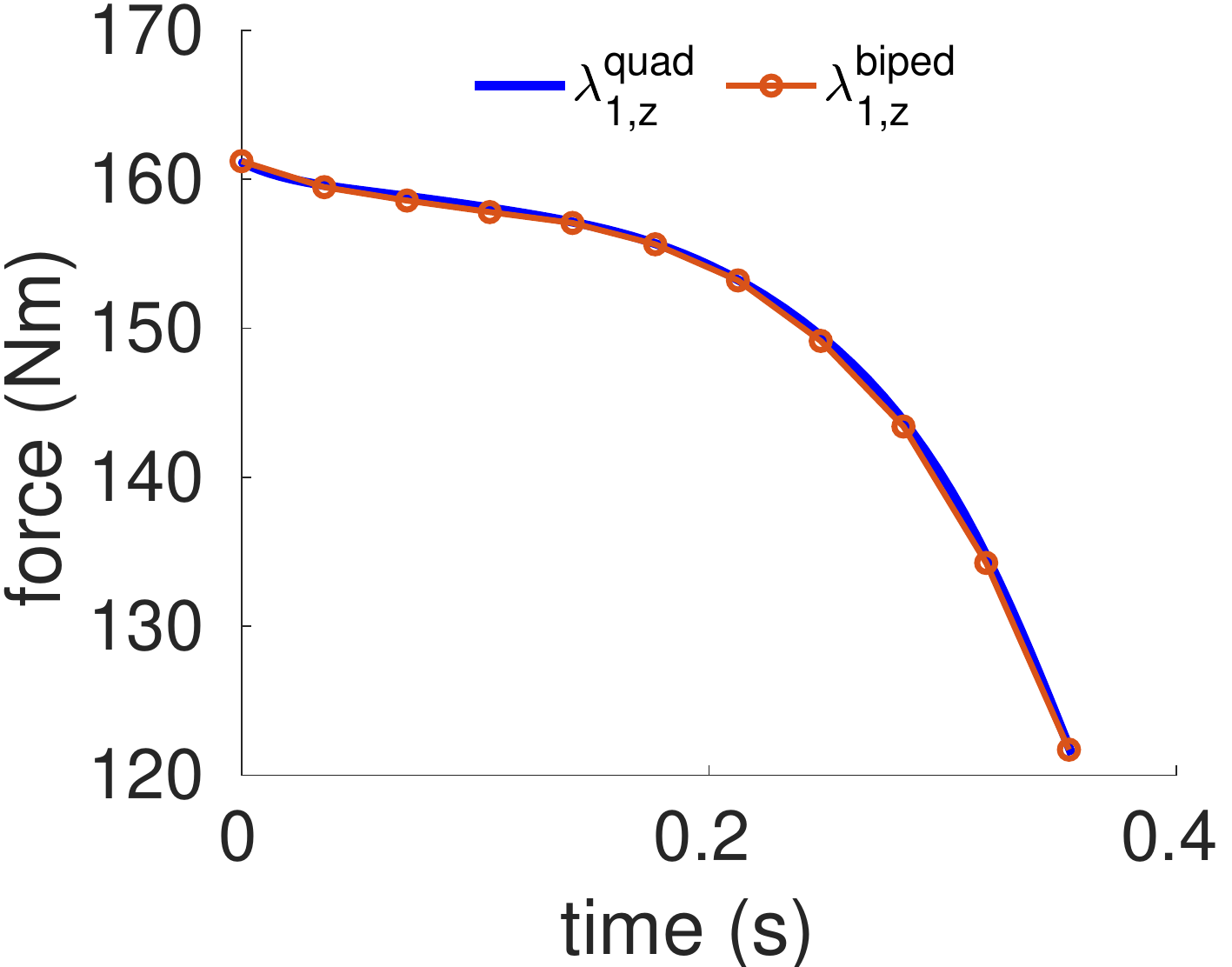}
		\vspace{-2mm}
	    \caption{{\small A comparison between the solution of bipedal walking dynamics obtained from the decomposition-based optimization and a simulated step of the full-order quadrupedal dynamics using the composed bipedal gaits; here \texttt{MATLAB ODE45} was used.}}
	    \label{fig:bi2quad}
	\end{center}
	\vspace{-7mm}
\end{figure*}

\subsection{Impact dynamics of the decomposed system}
With the continuous dynamics written as (OL-Dyn), we can similarly expand the impact dynamics \eqref{eq:impact} as:
\par\vspace{-3mm}{\small
\begin{align}
\Delta\defeq
\begin{cases}
    D_\rmf(\dot q_\rmf^+ - \dot q_\rmf^-) = J_{\rmf_2}^\top \Lambda_2 - J_c^\top \Lambda_c\\
    J_{\rmf_2} \dot{q}_\rmf^+ = 0 \\
    D_\rmr(\dot q_\rmr^+ - \dot q_\rmr^-) = J_{\rmr_1}^\top \Lambda_1 + J_c^\top \Lambda_c\\
    J_{\rmr_1} \dot{q}_\rmr^+ = 0 \\
    J_c( \dot q_\rmr^+ - \dot q_\rmf^+) = 0
\end{cases}
\label{eq:impactsss}
\end{align}}%
The proof is similar to that of continuous dynamics decomposition, thus omitted. On the ZD surface, where both of the bipedal systems (f) and (r) have zero tracking errors before and after the impact dynamics $\Delta$ (we will use an optimization algorithm to determine those gaits that are \textit{hybrid invariant}, \cite{GRABPL01,morrisTrans,Ames_Human_Inspired_IEE_TAC}), we have the correlation:
$
    q_\rmr^- = A q_\rmf^- + b, ~ 
    q_\rmr^+ = A q_\rmf^+ + b. 
$ 
Plug into \eqref{eq:impactsss} to get the impact dynamics of the decomposed system as:
\par\vspace{-3mm}{\small
\begin{align} 
    \begin{bmatrix}
        D_\rmf      & -J_{\rmf_2}^\top & 0 & J_c^\top \\
        J_{\rmf_2}  & 0  & 0 & 0\\
        D_\rmr A    & 0  & -J_{\rmr_1}^\top & -J_c^\top \\
        J_{\rmr_1}A & 0  & 0 & 0 
    \end{bmatrix}
    \begin{bmatrix}
        \dot q_\rmf^+ \\ \Lambda_1 \\ \Lambda_2 \\ \Lambda_c
    \end{bmatrix}
    &= 
    \begin{bmatrix}
        D_\rmf \dot q_\rmf^- \\ 0\\  D_\rmr A \dot q_\rmf^- \\ 0 
    \end{bmatrix}
    \label{eq:hyimpact1}
\end{align}}%
Note that although system \eqref{eq:hyimpact1} is an overdetermined system, removing the redundant equations is not desirable in practice, as it may result in an ill-posed problem. This issue can be more severe for robots with light legs. Moreover, the implicit optimization method in the latter section can solve this system accurately and efficiently.


\section{Decomposition-based optimization} \label{sec:opt}

Past work has investigated the formal analysis and controller design for the full-body dynamics of quadrupeds \cite{Hamed2019Dynamically, ma2019First}. Although we were able to produce trajectories that are stable solutions to the closed-loop multi-domain dynamics for walking, ambling, and trotting, the computational complexity makes realizing these methods difficult in practice: it typically takes minutes to generate a trajectory and hours to post-process the parameters to guarantee dynamic stability.  However, by using the dynamics decomposition method, we can produce bipedal walking gaits that can be composed to obtain quadrupedal locomotion while maintaining the efficiency of computing the lower-dimensional dynamics of bipedal robots. In this section, we detail this process using nonlinear programming (NLP).

Given the constrained bipedal dynamics (CL-Dyn-f) and the impact dynamics \eqref{eq:hyimpact1}, the target is to find a solution to the closed-loop dynamical system efficiently. The nonlinear program is formulated as: 
\par\vspace{-3mm}{\small 
\begin{align}
	\label{eq:opteqs}
	\min_{\mathbf{Z}} &\ \  \sum_{j=1}^{2N+1} \norm{\dot q_{\rmbf}}_2^2    \\
	\mathrm{s.t.} 	  
&\ \   \textbf{C1}.\ \text{dynamics (CL-Dyn-f)}                  \hspace{10mm} &j = 1,3,...2N+1  \notag \\
&\ \   \textbf{C2}.\ \text{collocation constraints       }                     &j = 2,4,...2N \notag\\
&\ \   \textbf{C3}.\ \text{impact dynamics} \eqref{eq:hyimpact1}               &j = 2N+1 \notag\\
&\ \   \textbf{C4}.\ \text{periodic continuity}                                &j = 1, 2N+1 \notag\\
&\ \   \textbf{C5}.\ \text{physical feasibility}                               &j = 1,2,...2N+1  \notag
\end{align}}%
with the following notation:
$2N+1=11$ is the total number of collocation grids; 
the decision variable is defined as
\par\vspace{-4mm}{\small 
\begin{align*}
\mathbf{Z}=(\alpha, t^j, q_\rmf^j, \dot q_\rmf^j,  u_\rmf^j,  
\lambda_1^j, \lambda_2^j, \lambda_c^j, 
\Lambda_1^j, \Lambda_2^j, \Lambda_c^j);
\end{align*}}%
and $\alpha\in\R^{36}$ \blue{are the coefficients for the Be\'zier polynomial} that defines the desired trajectory $\mathcal{B}_\rmf(t)$; $\square^j$ is the corresponding quantities at time $t_j$ with $t^{2N+1}=T$.
In short,  the cost function is to minimize the body's vibration rate to achieve a more static torso movement. The constraints \textbf{C1}-\textbf{C3} solve the hybrid dynamics of bipedal robots subject to external forces. Details regarding the numerical optimization can be found in \cite{hereid20163d}. In particular, the Hermite-Simpson collocation formulation can be found in equations (C1,C2) in \cite{hereid20163d}. Here, the periodic continuity constraint \textbf{C4} enforces state continuity through an edge, i.e., the post-impact states $q^+, \dot q^+$, are equivalent to the initial states $q^1, \dot q^{1}$. Therefore, the resultant trajectory is a periodic solution to the bipedal dynamics. 
\textbf{C5} imposed some feasibility conditions on the dynamics, including torque limits $\norm{u_i}_{\infty} \leq 50$, joint feasible space $(q_i, \dot q_i)\in\mathcal{X}$, foot clearance and the friction pyramid conditions. Note that we posed these constraints conservatively to reduce the difficulties implementing the optimized trajectories in experiments.

To solve the optimization problem \eqref{eq:opteqs} efficiently, we used a toolbox FROST \cite{Hereid2018Rapid,hereid2018dynamic}, which parses a hybrid control problem as a NLP based on direct collocation methods, in particular, Hermite-Simpson collocation. It is worthwhile to mention that a critical reason for the high efficiency of FROST comes from the implicit formulation of the dynamics. Matrix inversion is avoided in every step due to its computational complexity: $\mathcal{O}(n^3)$, with $n$ the dimension of a matrix. Inspired by this, we remark the dynamics decomposition method proposed in this paper also only used differential algebra equations (DAEs) instead of ordinary differential equations (ODEs), which requires matrix inversion both for the inertia matrix and the closed-loop controller formulation. 

Once the optimization \eqref{eq:opteqs} converged to a set of parameters $\alpha$ for the front bipedal robots' walking gait $\mathcal{B}_\rmf(t)$, we can use \eqref{eq:bi2quadmap} to obtain the trajectory for the rear biped $\mathcal{B}_\rmr(t)$ and then recompose them to get the parameters for the quadrupedal locomotion. For validation, we simulated an ambling step of the quadrupedal dynamics using the composed bipedal gaits. As shown in \figref{fig:bi2quad}, we have the joint angles and constraint wrench (ground reaction force) on toe1 $\lambda_{1,z}$, and toe2 $\lambda_{2,z}$ of the quadruped matched with those corresponding external force to the bipedal dynamics.

\begin{table}[b]
\vspace{-4mm}
\caption{{\small Computing performance of gait generation. This is performed on a Linux machine with an i7-6820HQ CPU @$2.70$ GHz and $16$ GB RAM.}} 
\centering 
    \vspace{-1.2mm}
        \begin{tabular}{|c|c|c|c|c|c|} \hline
             \textbf{Behaviors}  & gait1    & gait2 & gait3 & gait4 & amble \\ \hline
        frequency (Hz)           & 2.5      & 2.3   & 2.2   & 2.6   & 2.83  \\ \hline
        clearance (cm)           & 11       & 12    & 15    & 13    & 13    \\ \hline
        \# of iterations         & 96       & 122   & 98    & 46    & 147   \\ \hline
        time of IPOPT (s)        & 1.60     & 2.10  & 1.62  & 0.81  & 2.59  \\ \hline
        time of evaluation (s)   & 1.94     & 3.24  & 2.10  & 0.94  & 2.86  \\ \hline
        NLP time(s)              & 3.54     & 5.34  & 3.72  & 1.75  & 5.45  \\ \hline
    \end{tabular}    
    \label{t:table1}
\end{table}


We now take advantage of the efficient, decomposition-based optimization to generate several walking patterns for the front biped, then recompose them to obtain quadrupedal stepping-in-place behaviors. By adjusting the constraint bounds in the NLP \eqref{eq:opteqs}, such as the upper and lower bound $T_{\max}, T_{\min}$ of time duration $T$, or the the bounds of the nonstance foot height $\delta_{\min} \leq h_{\text{nsf},z}(q_\rmf) \leq \delta_{\max}$, we can obtain gaits with different stepping frequency and foot clearance. Further, we remove the constraint that the nonstance foot lands at the origin to generate a \textit{diagonally ambling} gait with a speed of $0.35$ m/s. See \figref{fig:tiles} for the tiles of these gaits.
%
The result of the methods presented is the ability to generate quadrupedal gaits rapidly. We benchmark the performance by considering computing speed for each of the quadrupedal locomotion patterns generated, as is shown in Table \ref{t:table1}. In summary, with the objective tolerance and equality constraint tolerance configured as $10^{-8}$ and $10^{-5}$ respectively, we have the average computation time as $3.96$ second, and time per iteration averages $0.039$ second. 
In comparison with the regular full-model based optimization methods from \cite{ma2019First}, the decomposition-based optimization is an order of magnitude faster. 


\begin{figure}[!t]
    \centering
    \vspace{1mm}
	\includegraphics[width=0.494\textwidth]{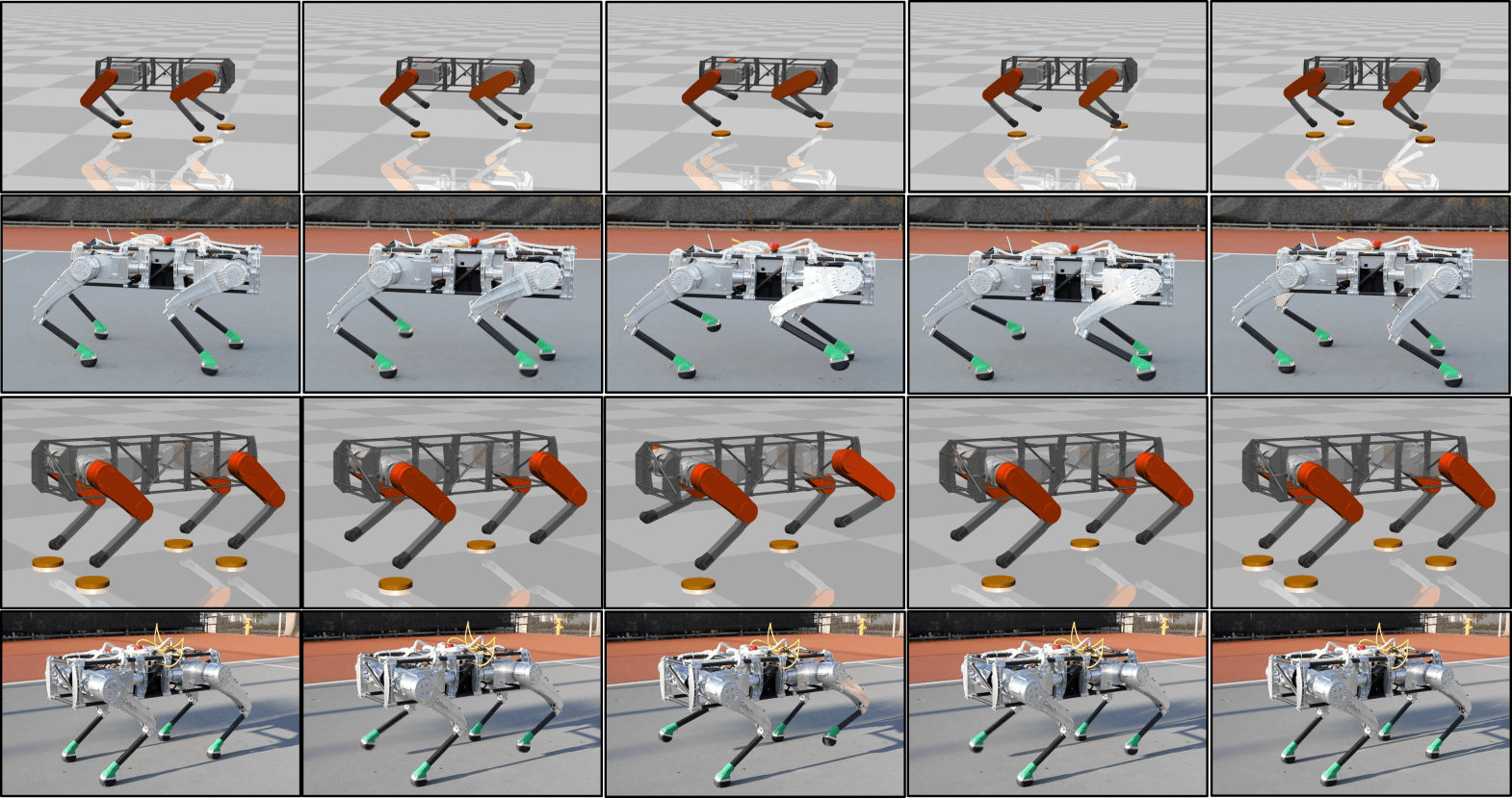} 
    \caption{{\small
    Comparison between MuJoCo simulation (animated) and experiments. The upper two are for stepping in place, gait4; the lower two are logged for a full step of the ambling gait.}}
    \label{fig:tiles}
    \vspace{-2.6mm}
\end{figure}

\begin{figure}[t]
	\centering
	\includegraphics[width=0.115\textwidth]{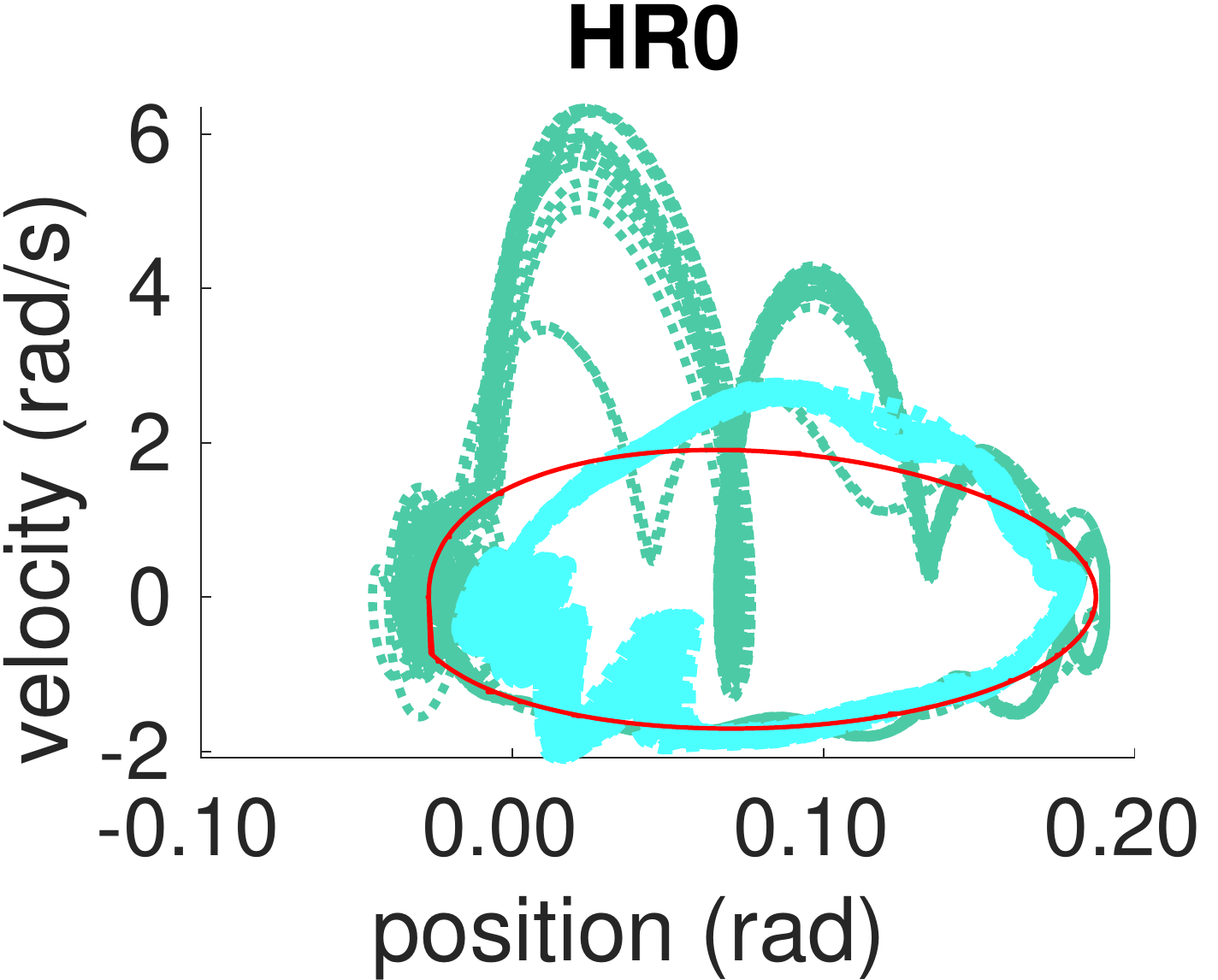}
	\includegraphics[width=0.115\textwidth]{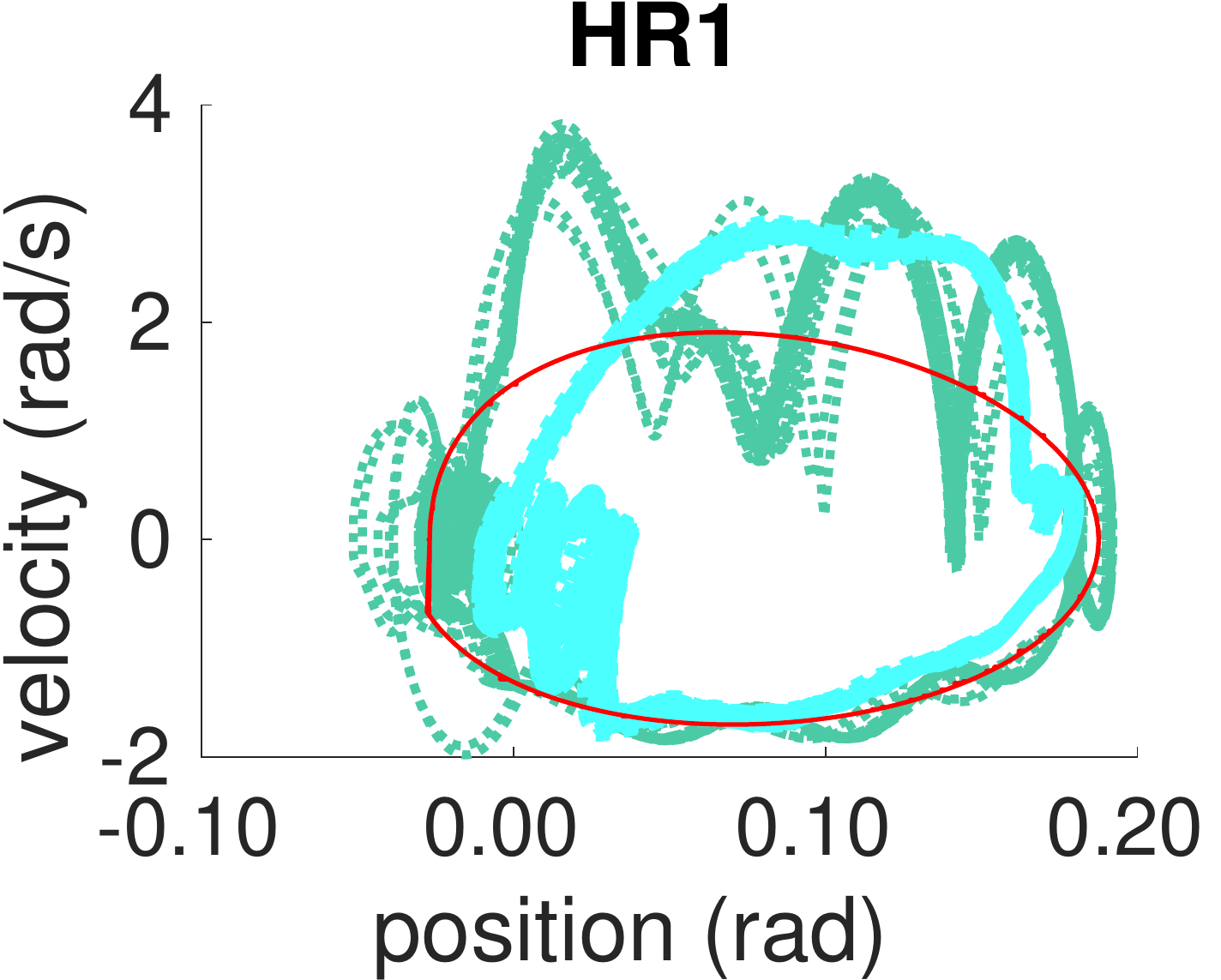}
	\includegraphics[width=0.115\textwidth]{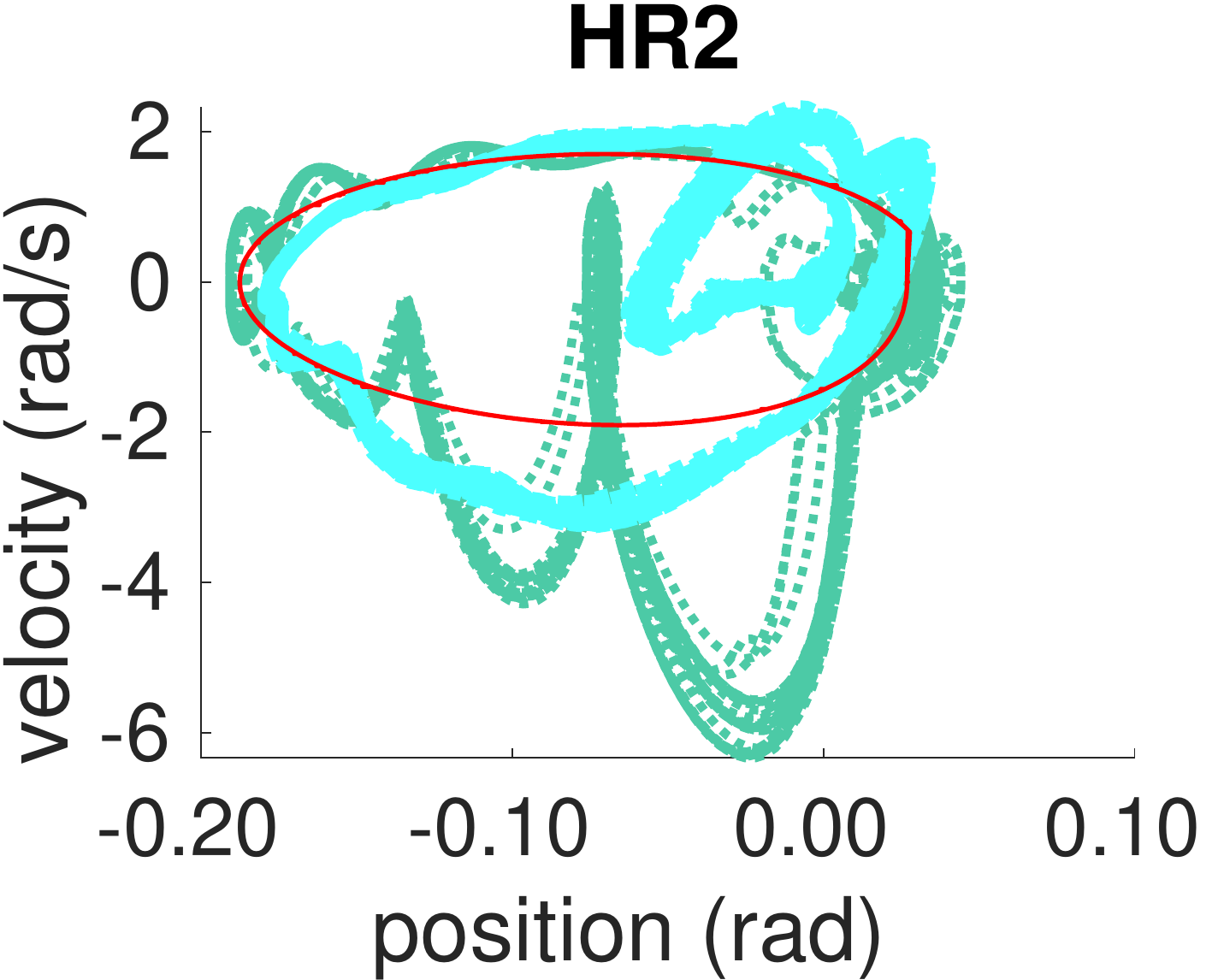}
	\includegraphics[width=0.115\textwidth]{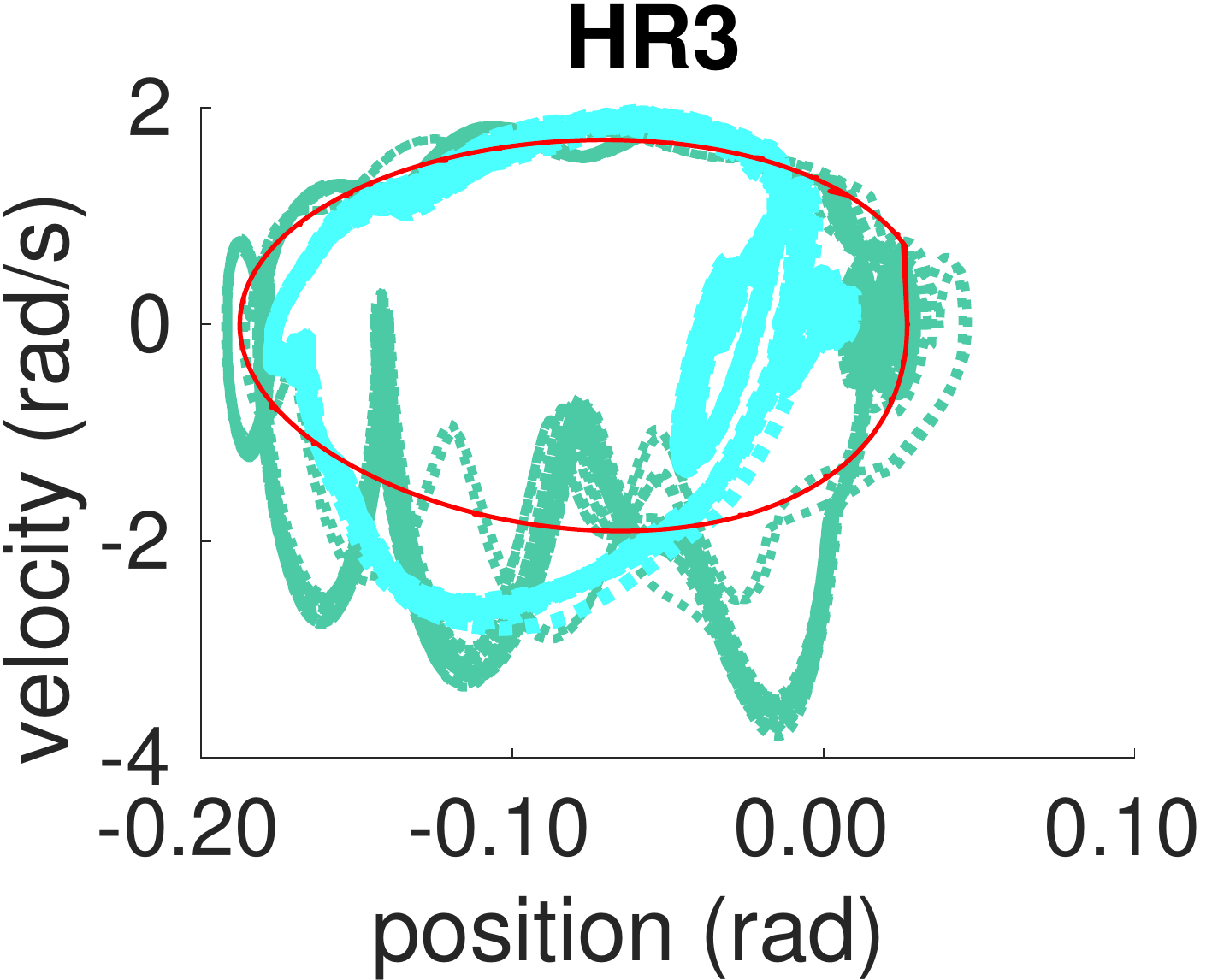}
	\\
	\includegraphics[width=0.115\textwidth]{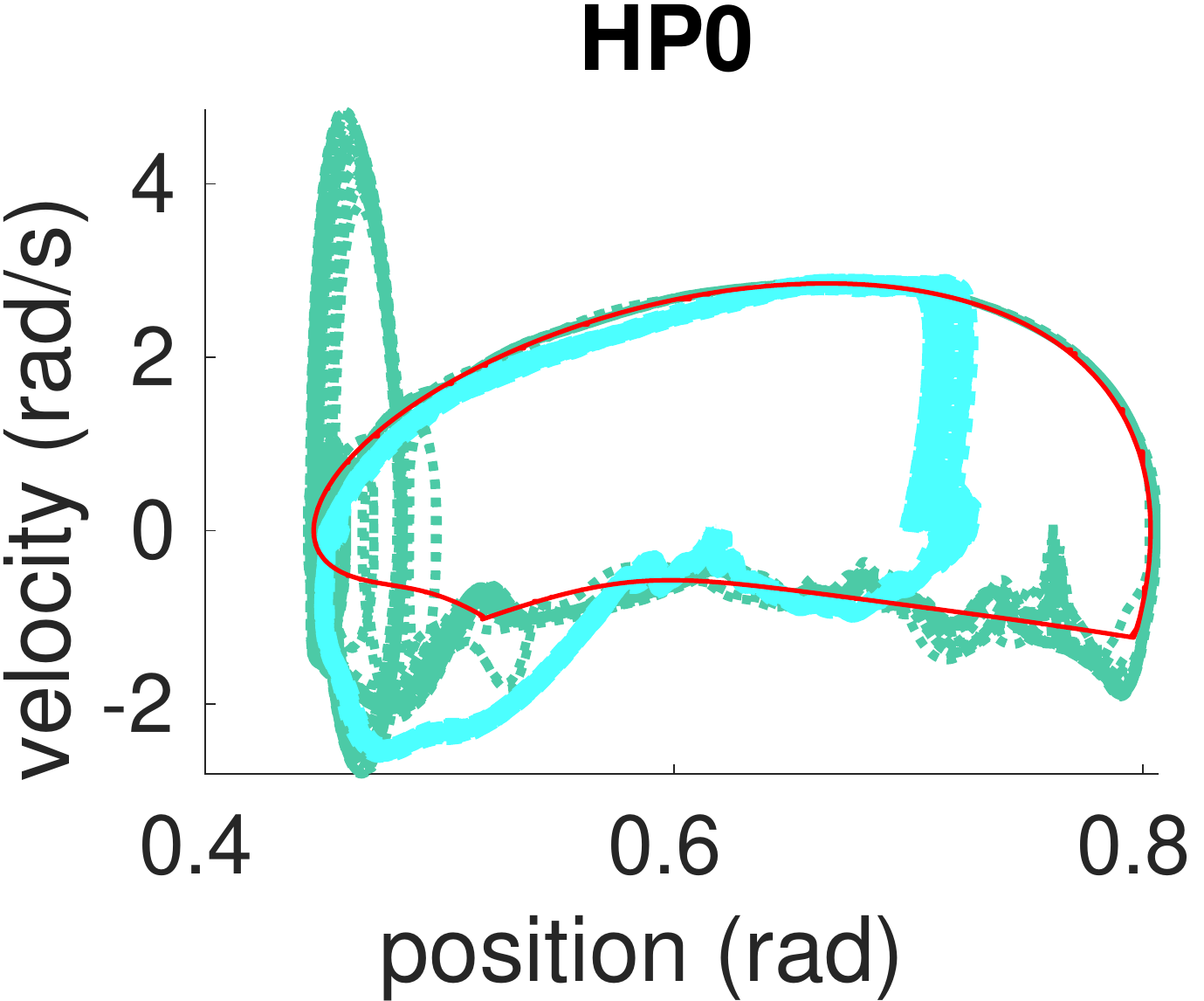}
	\includegraphics[width=0.115\textwidth]{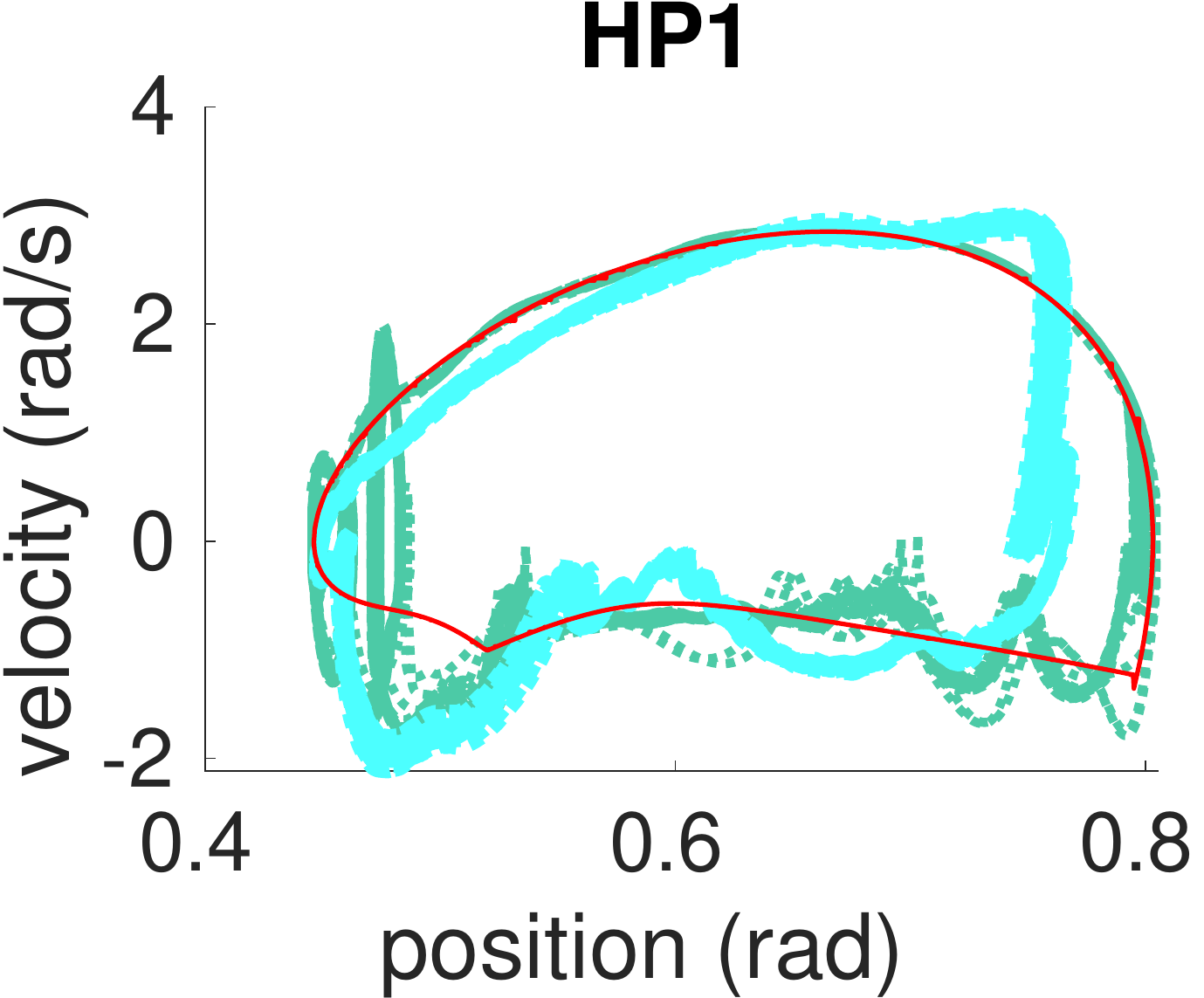}
	\includegraphics[width=0.115\textwidth]{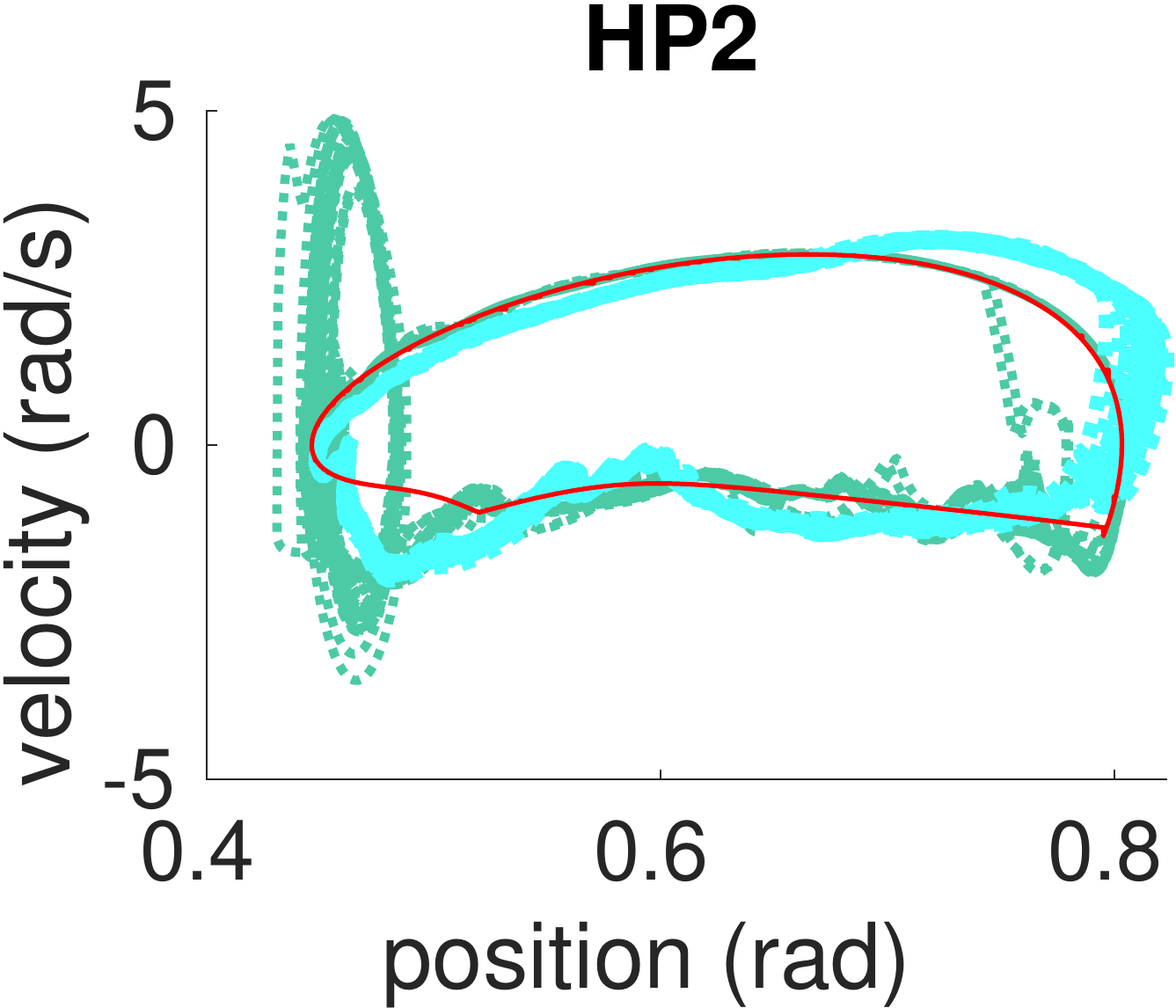}
	\includegraphics[width=0.115\textwidth]{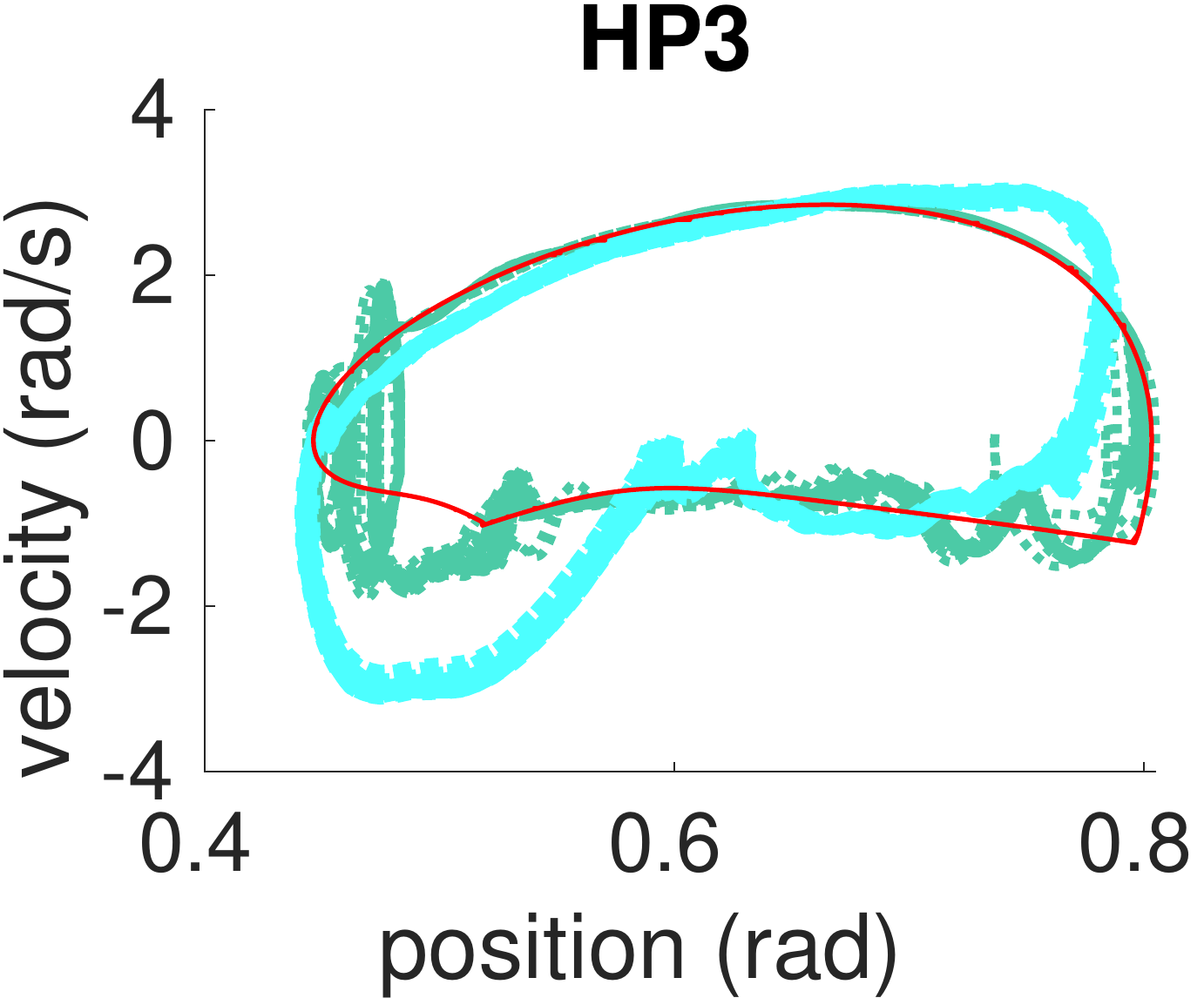}
	\\
	\includegraphics[width=0.115\textwidth]{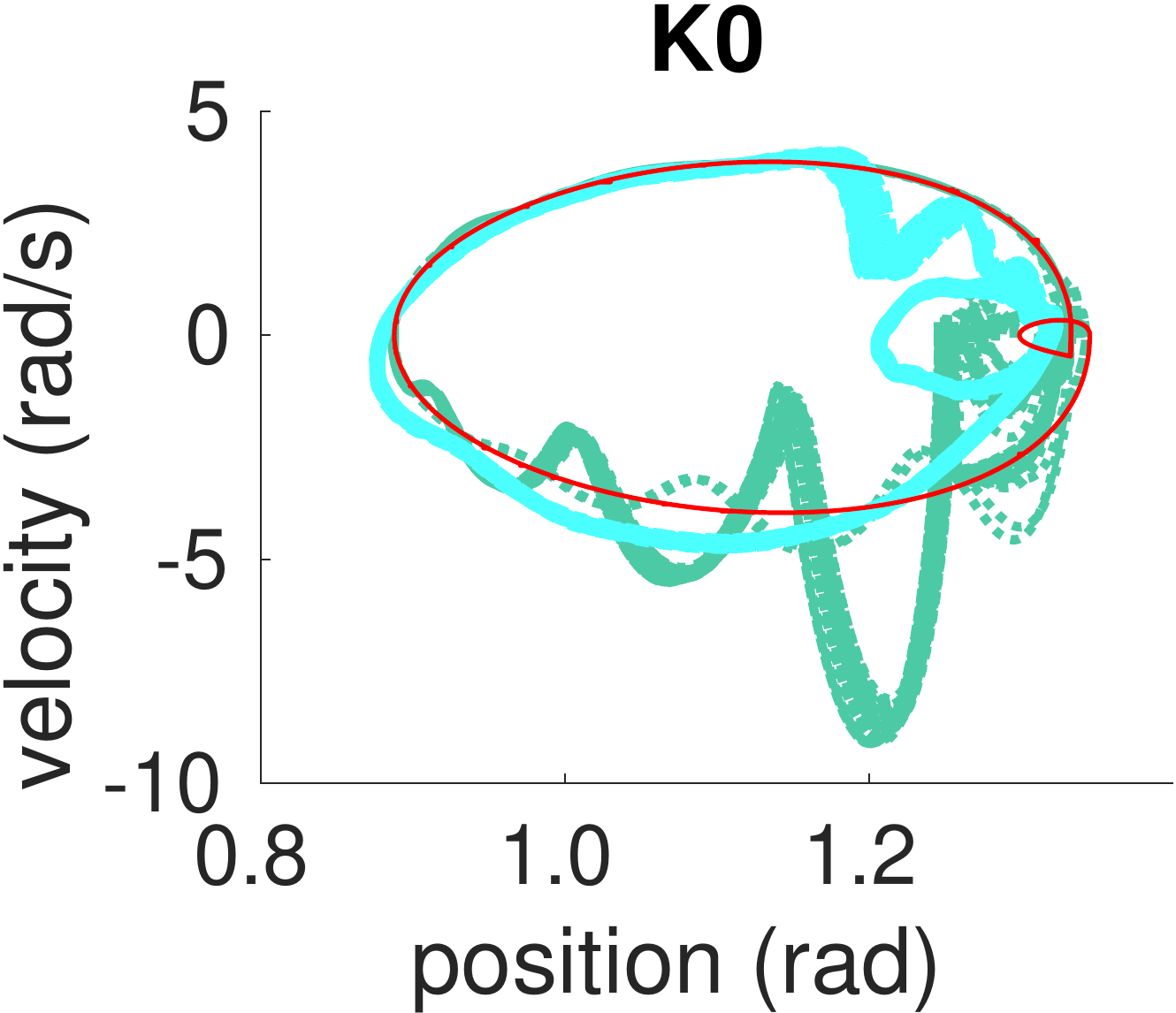}
	\includegraphics[width=0.115\textwidth]{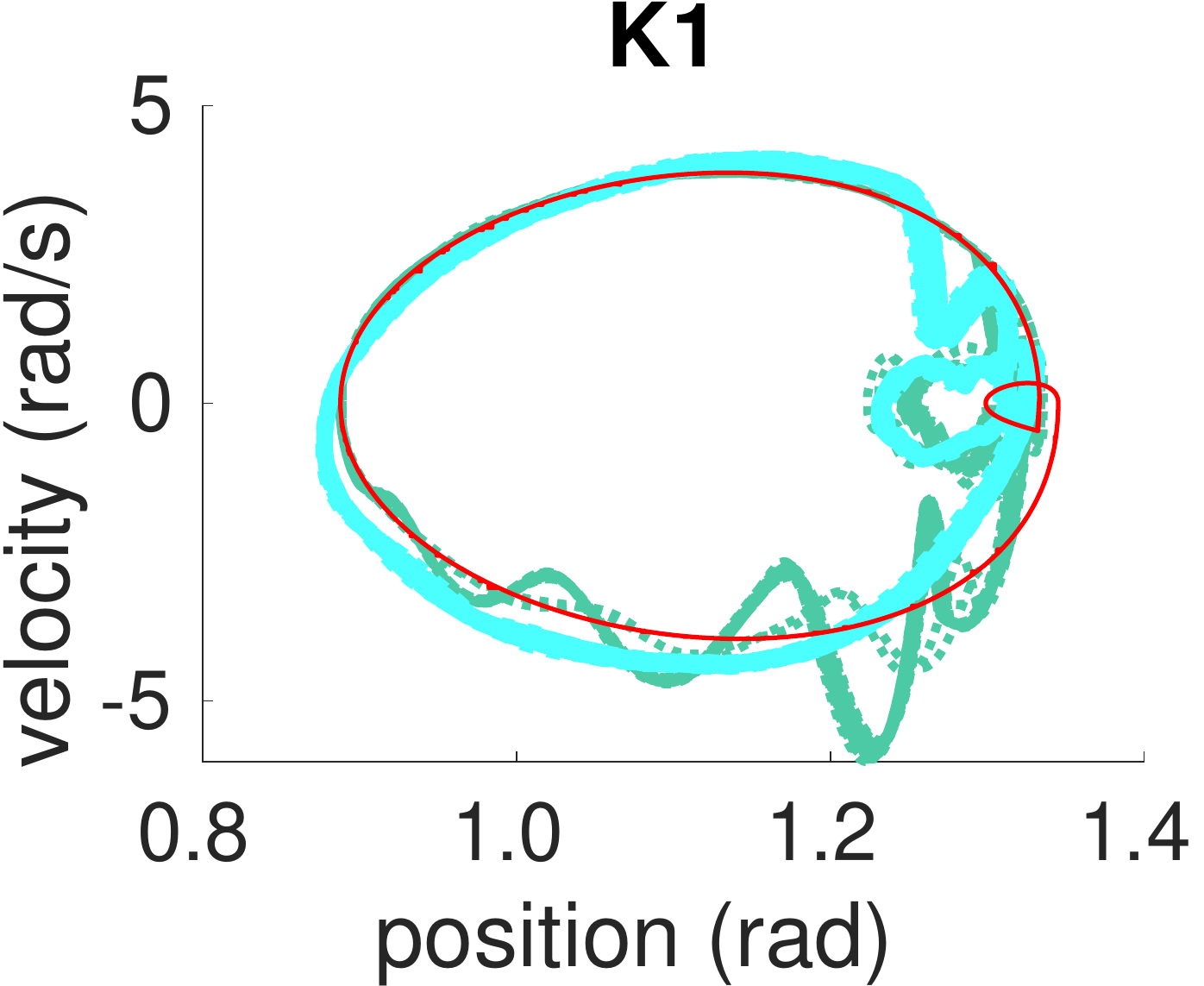}
	\includegraphics[width=0.115\textwidth]{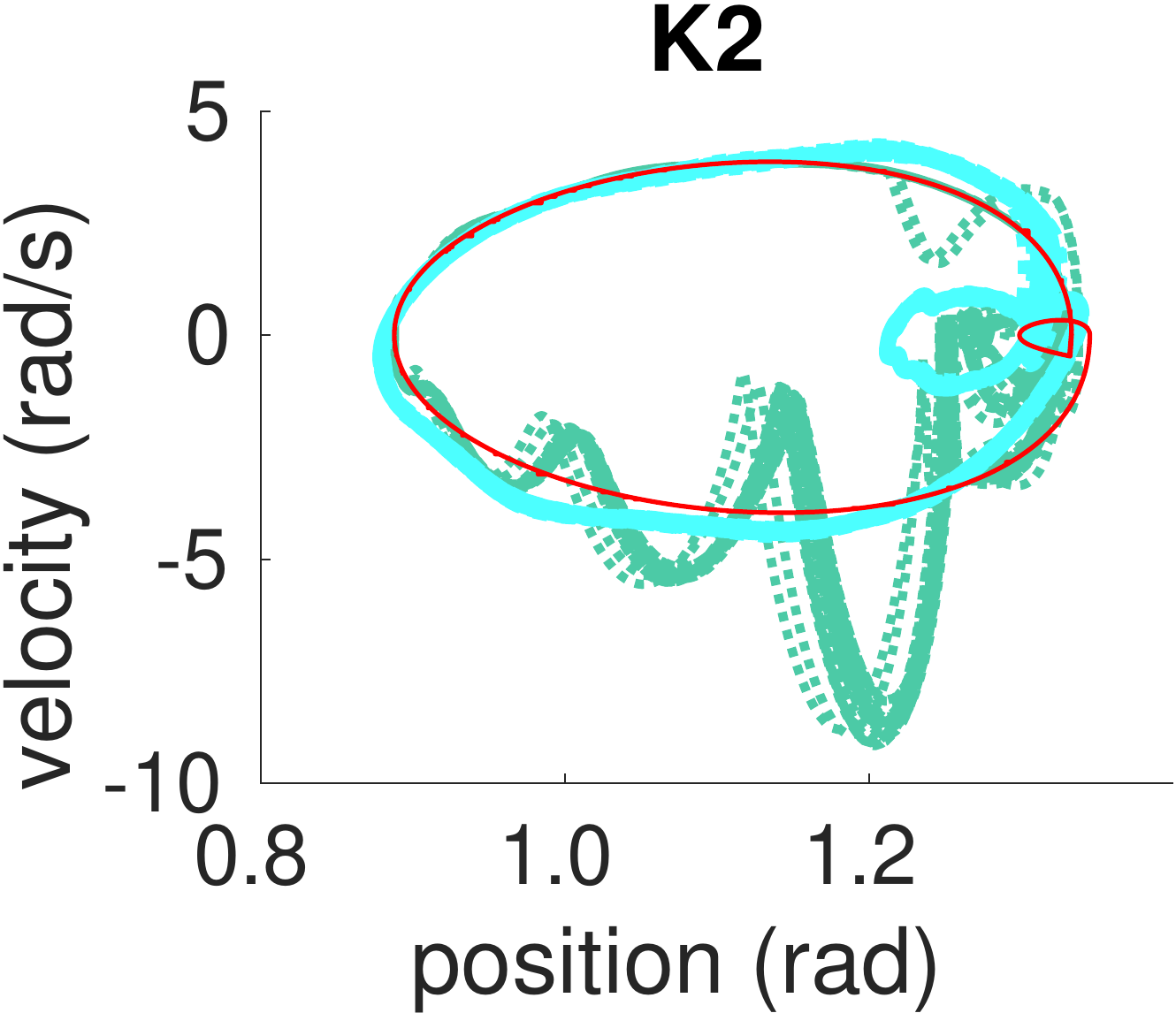}
	\includegraphics[width=0.115\textwidth]{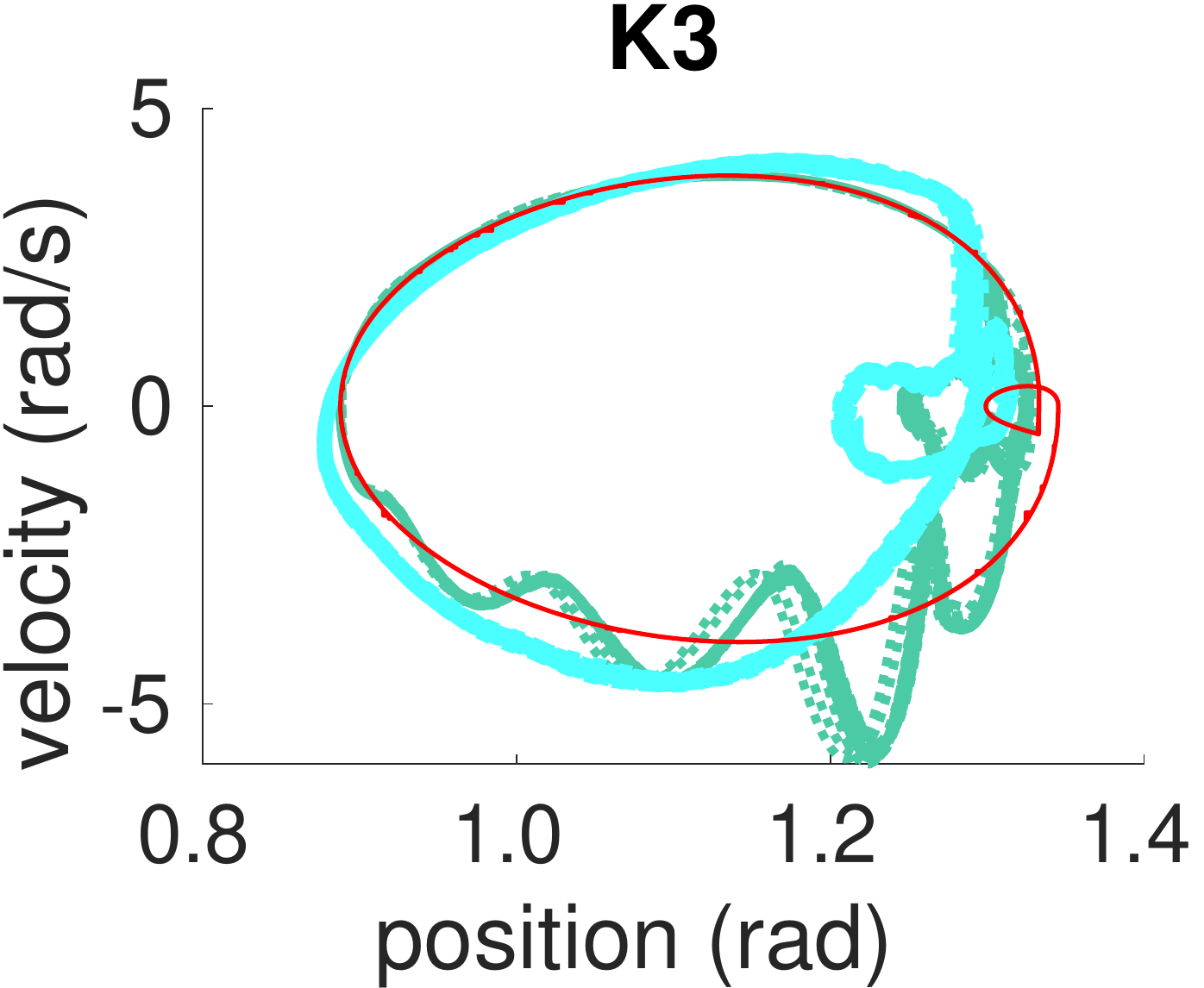}
	\vspace{-1mm} 
	\caption{{\small Tracking performance of the optimal ambling gait (in red) vs. the MuJoCo simulated result (in green) vs. the experimental data (in cyan) in the form of phase portrait using 18 seconds' data. HR: hip roll joint; HP: hip pitch joint; K: knee joint.}}
	\label{fig:walk2expsim}
	\vspace{-5mm}
\end{figure}

\section{Simulation and experiments}


One of the motivations for realizing rapid gait generation using the full-body dynamics of the quadruped, i.e., without model simplifications, is to allow for the seamless translation of gaits from theoretical simulation to hardware. In this context, we first validated the dynamic stability of the gaits produced by the decomposition-based optimization problem using a third party physics engine --- MuJoCo. These gaits include a diagonally symmetric ambling and four stepping in place behaviors. Then we conducted experiments, walking on a  a outdoor tennis court, using the same control law as that in simulation in outdoor environments. In particular, we used a PD approximation of the input-output linearizing controllers to track the time-based trajectories given by the optimization, 
\par\vspace{-3mm}{\small 
\begin{align}
    \label{eq:PD}
    u(q_a, \dot q_a, t) = 
    -k_1\big(\dot y_a - \mathcal{\dot{B}}(t) \big) -k_2\big(y_a - \mathcal{B}(t) \big). 
\end{align}}%
Note that the event functions (switching detection) are also given by the optimized trajectories, meaning the walking controller switches to the next step when $t=T$. We report that for all given optimal gaits, the PD gains are picked as $k_p = 230,\ 230,\ 300, k_d = 5$ for the hip roll, hip pitch, knee joints, respectively. The averaged absolute joint torque inputs are logged in Table \ref{t:exp}, all of which are well within the hardware limitations. The tracking of the ambling gait in simulation and experiment are shown in \figref{fig:walk2expsim}. 

\begin{table}[t] 
\centering 
    \caption{{\small Average torque inputs in experiments and simulations.}}
    \vspace{-1.2mm} 
        \begin{tabular}{|l|l|l|l|l|l|} \hline
               \textbf{Experiments}         & gait1    & gait2 & gait3 & gait4 & amble \\ \hline 
        $\bar{u}_{\text{HR}}$(N$\cdot$m)    & 5.04,    & 4.83  & 4.16  & 5.14  & 7.11  \\ \hline 
        $\bar{u}_{\text{HP}}$(N$\cdot$m)    & 3.65     & 5.24  & 5.26  & 3.77  & 6.28  \\ \hline 
    $\bar{u}_{\text{K}}$(N$\cdot$m)         & 16.45    & 16.50 & 16.86 & 16.95 & 18.36 \\ \hline 
               \textbf{MuJoCo}              &          &       &       &       &       \\ \hline 
        $\bar{u}_{\text{HR}}$(N$\cdot$m)    & 7.80     & 9.23  & 10.27 & 8.68  & 8.06  \\ \hline 
        $\bar{u}_{\text{HP}}$(N$\cdot$m)    & 6.78     & 9.14  & 10.71 & 6.64  & 7.27  \\ \hline 
        $\bar{u}_{\text{K}}$(N$\cdot$m)     & 18.49    & 18.38 & 18.45 & 18.61 & 19.03 \\ \hline 
    \end{tabular}
    \label{t:exp}
    \vspace{-4mm}
\end{table}


The result is that the Vision 60 quadruped can step and amble in an outdoor tennis court in a sustained fashion. Importantly, this is without any add-on heuristics 
and achieved by only uploading different gait parameters $\alpha$ for each experiment (obtained from the different NLP optimization problems with different constraints
). 
See \cite{amble_video} for the video of Vision 60 in both simulation and experiments. 
As demonstrated in the video, we remark that the proposed method has rendered a good level of robustness against rough terrain with slopes, wet dirt and surface roots. Hence periodic stability has been obtained in both simulation and experiment. \figref{fig:tiles} shows a side to side comparison of the simulated amble and experimental snapshots. 
In addition, it is interesting to note that time-based control law \eqref{eq:PD} normally does not provide excellent robustness against uncertain terrain dynamics, due to its open-loop nature. However, the fact that all of the trajectory-based controllers achieved dynamic stability in simulations and experiments with an unified control law speaks to the benefits of generating gaits using the full-body dynamics of the quadruped: even with an open-loop controller that does not leverage heuristics, the quadruped remains stable. 

\section{Conclusion} \label{sec:conclusion}

In this paper, we decomposed the full-body dynamics of a quadrupedal robot --- the Vision 60 with 18 DOF and 12 inputs --- into two lower-dimensional bipedal systems that are subject to external forces. We are then able to solve the constrained dynamics of these bipeds quickly through the HZD optimization method, FROST, wherein the gaits can be recomposed to achieve locomotion on the original quadruped. 
The result is the ability to generate walking gaits rapidly. Specifically, by changing a constraint, we can produce different bipedal and, thus, quadrupedal walking behaviors from stepping to ambling in $3.9$ seconds on average. Furthermore, the implementation in simulation and experiments used a single simple controller, without the need for additional heuristics.

Without sacrificing the model fidelity of the full-body dynamics of the quadruped, the ability to exactly decompose these dynamics into equivalent bipedal robots makes it possible to rapidly generate gaits that leverage the full-order dynamics of the quadruped. Importantly, this allows for the rapid iteration of different gaits necessary for bringing quadrupeds into real-world environments. Moreover, the fact that these gaits can be generated on the order of seconds suggests that with code optimization on-board and real-time gait generation may be possible soon.  The goal is to ultimately use this method to realize a variety of different dynamic locomotion behaviors on quadrupedal robots. 


\addtolength{\textheight}{-6.4cm}
\bibliographystyle{abbrv}
\bibliography{cite}

\end{document}